%% file: main-extended.tex
\documentclass{article}
\usepackage{ijcai23}

\usepackage{times}
\usepackage{soul}
\usepackage{url}
\usepackage[hidelinks]{hyperref}
\usepackage[utf8]{inputenc}
\usepackage[small]{caption}
\usepackage{graphicx}
\usepackage{amsmath}
\usepackage{amsthm}
\usepackage{booktabs}
\usepackage{algorithm}
\usepackage{algorithmic}
\usepackage[switch]{lineno}

\usepackage{tikz}
\usetikzlibrary{trees}
\usetikzlibrary{shapes,arrows}
\usetikzlibrary{intersections,positioning,calc}
\usetikzlibrary{decorations.pathreplacing,calc,decorations.pathmorphing}
\usepackage{tikz-qtree}

\usepackage{stmaryrd} 
\usepackage{xspace}
\usepackage{thm-restate}
\usepackage{cancel}
\usepackage[disable]{todonotes}
\usepackage{tcolorbox}
\usepackage{csquotes} 
\usepackage{multirow}
\usepackage{bm}
\usepackage{enumitem}
\usepackage{mathtools}
\usepackage[usestackEOL]{stackengine}
\usepackage{colortbl}

\newtheorem{example}{Example}
\newtheorem{theorem}{Theorem}

\newtheorem{definition}{Definition}
\newtheorem{corollary}{Corollary}
\newtheorem{lemma}{Lemma}

\newtheorem{remark}{Remark}

\title{Efficient Computation of General Modules for \ALC Ontologies\\
(Extended Version)
}

\author{
Hui Yang$^1$
\and
Patrick Koopmann$^2$\and
Yue Ma$^{1}$\And
Nicole Bidoit$^1$
\affiliations
$^1$LISN, CNRS, Université Paris-Saclay\\
$^2$
Vrije Universiteit Amsterdam, The Netherlands
\emails
yang@lisn.fr,
p.k.koopmann@vu.nl,
ma@lisn.fr, 
nicole.bidoit@lisn.fr
}

\begin{document}

\include{macros}
\maketitle

\begin{abstract}
 We present a method for extracting general modules for ontologies formulated in
the description logic \ALC. A module for an ontology is an ideally substantially
smaller ontology that preserves all entailments for a user-specified
set of terms.
As such, it has applications such as ontology reuse and ontology analysis.
Different from classical modules, general modules may use axioms not
explicitly present
in the input ontology, which allows for additional conciseness.
So far, general modules have only been investigated for lightweight
description logics.
We present the first work that considers the more expressive
description logic \ALC. In
particular, our contribution is a new method based on uniform
interpolation supported
by some new theoretical results.
Our evaluation indicates that our general modules are often smaller
than classical modules and
uniform interpolants computed by the state-of-the-art, and compared
with uniform interpolants,
can be computed in a significantly shorter time.
Moreover, our method can be used for, and in fact improves, the computation
of uniform interpolants and classical modules.
\end{abstract}

%
%
%

\input{introduction}
Detailed proofs of the results can be found in the appendix.

\input{table-running-example}

\input{preliminaries_ijcai}

\input{ontology_normalization}

\input{role_elimination}

\input{extracting_general_modules}

\input{extracting_deductive_modules}

\input{evaluation}

\input{conclusion}

\section*{Acknowledgements}
Hui Yang, Yue Ma and Nicole Bidoit are funded by the BPI-France (PSPC AIDA: 2019-PSPC-09) and ANR (EXPIDA: ANR-22-CE23-0017). Patrick Koopmann is partially funded by DFG Grant 389792660 as part of TRR~248---CPEC (see \url{https://perspicuous-computing.science})

\bibliographystyle{named}
\bibliography{bibliography}

\appendix

\input{appendix}

\end{document}

%% file: macros.tex
\newcommand{\ra}{\rightarrow}
\newcommand{\ourMod}{general module}
\newcommand{\OurMod}{General module}
\newcommand{\sig}{\textit{sig}\xspace}
\newcommand{\sigC}{\textit{sig}_\mathsf{C}\xspace}
\newcommand{\sigR}{\textit{sig}_\mathsf{R}\xspace}
\newcommand{\dep}{\textit{dep}\xspace}

\newcommand{\patrick}[2][inline]{\tttodo{#1}{Patrick}{orange}{#2}}
\newcommand{\hui}[2][inline]{\tttodo{#1}{Hui}{green}{#2}}
\newcommand{\nicole}[2][inline]{\tttodo{#1}{Nicole}{pink}{#2}}
\newcommand{\yue}[2][inline]{\tttodo{#1}{Yue}{yellow}{#2}}
\newcommand{\tttodo}[4]{\ifthenelse{\equal{#1}{inline}}{\todo[inline,
author=#2, color =
#3]{#4}}{\todo[color=#3]{#2: #4}}}

\newcommand{\hideThisPart}[1]{}

\newcommand{\yang}[1]{{\color{blue} #1}}
\newcommand{\NF}[1]{{\color{purple} #1}}
\newcommand{\red}[1]{{\color{red} #1}}
\newcommand{\Omc}{\ensuremath{\mathcal{O}}\xspace}
\newcommand{\Imc}{\ensuremath{\mathcal{I}}\xspace}
\newcommand{\Jmc}{\ensuremath{\mathcal{J}}\xspace}
\newcommand{\Mmc}{\ensuremath{\mathcal{M}}\xspace}
\newcommand{\starM}{$\top\!\bot^\ast$-module}
\newcommand{\ALC}{\ensuremath{\mathcal{ALC}}\xspace}
\newcommand{\ALCH}{\ensuremath{\mathcal{ALCH}}\xspace}
\newcommand{\EL}{\ensuremath{\mathcal{EL}}\xspace}
\newcommand{\ELH}{\ensuremath{\mathcal{ELH}}\xspace}

\newcommand{\NC}{\ensuremath{\mathsf{N_C}}\xspace}
\newcommand{\NR}{\ensuremath{\mathsf{N_R}}\xspace}
\newcommand{\ND}{\ensuremath{\mathsf{N_D}}\xspace}

\newcommand{\quant}{{\mathsf{Q}}}

\newcommand{\ExpTime}{\ensuremath{\textsc{ExpTime}}\xspace}
\newcommand{\TwoExpTime}{\ensuremath{\textsc{2ExpTime}}\xspace}

\newcommand{\Fame}{\ensuremath{\textsc{Fame}}\xspace}
\newcommand{\Lethe}{\ensuremath{\textsc{Lethe}}\xspace}

\newcommand{\GM}{\texttt{gm}\xspace}
\newcommand{\DM}{\texttt{dm}\xspace}
\newcommand{\GMLETHE}{\texttt{gmLethe}\xspace}
\newcommand{\minM}{\texttt{minM}\xspace}

\newcommand{\conE}{\texttt{conE}}
\newcommand{\rolE}{\texttt{rolE}}

\newcommand{\cl}{\textit{cl}}
\newcommand{\gm}{\textit{gm}}
\newcommand{\Dm}{\textit{dm}}
\newcommand{\Res}{\textit{Res}}
\newcommand{\clSigmaPart}{\ensuremath{\cl_\Sigma(\Omc)}\xspace}
\newcommand{\deSigmaPart}{\ensuremath{\mathcal{D}_{\Sigma}(\mathcal{O})}\xspace}

\newcommand{\sigmaForm}{role isolated\xspace}
\newcommand{\SigmaForm}{Role Isolated\xspace}

\newcommand{\unbounded}{\ensuremath{{Out}}\xspace}
\newcommand{\OsplitForm}{\ensuremath{\textit{RI}_\Sigma(\mathcal{O})}\xspace}

\newcommand{\longont}{\ensuremath{\lVert\Omc\rVert}}
\newcommand{\longcl}{\ensuremath{\lVert\cl(\Omc)\rVert}}

\newcommand{\ourM}{\textsc{GeMo}\xspace}

%% file: introduction.tex
\section{Introduction}\label{sec:intro}

Ontologies are used to formalize terminological knowledge
in many domains such as biology, medicine and the Semantic Web.
Usually, they contain a set of statements (axioms) about concept and role
names (unary and binary predicates).
Using a formalization based on description logics (DLs) allows DL reasoners to
infer implicit information from an ontology.
Modern ontologies are often large
and complex, which can make ontology engineering challenging. For example, as
of 3~January 2023, the medical ontology SNOMED CT~\cite{SNOMED}, used in the
health-case systems of many countries, formalizes over 360,000 concepts, and the
BioPortal repository of ontologies from the bio-medical domain~\cite{BIOPORTAL}
currently hosts 1,043 ontologies that use over 14 million concepts. Often, one
is not interested in the entire content of an ontology, but only in a fragment,
for instance if one wants to reuse content from a larger ontology for a more
specialized application, or to analyse what the ontology states about a given
set of names. In particular, this set of names would form a \emph{signature} $\Sigma$,
%
got which we want to compute an
ontology $\Mmc$ that captures all the logical entailments of the original
ontology $\Omc$ expressed using only the names in 
$\Sigma$. Our aim is to compute such an \Mmc that 
is as simple as possible, in terms of number and size of axioms.
%
This problem has received a lot of
attention in the past years, 
 and for an $\Mmc$ satisfying those requirements, the common approaches are
 \emph{modules} and \emph{uniform interpolants}.

\input{fig-tikz.tex}
\emph{Modules} are \emph{syntactical subsets} of the ontology $\Omc$
that preserve entailments within a given signature.
There is a variety of notions of modules and properties they can
satisfy  
that have been investigated in the literature~
\cite{grau2008modular,Konev-MODULES}.
\emph{Semantic modules} 
preserve all models of the ontology modulo the given signature $\Sigma$.
This makes them undecidable already for light-weight DLs such as
\EL~\cite{MODULE_UNDECIDABLE}, which is why existing methods often compute
approximations of minimal modules~\cite{GatensMODULES,RomeroModules}. A popular
example are \emph{locality-based modules}, 
which
can be computed in a very short time~\cite{grau2008modular}.
However, 
locality-based modules can be comparatively large, even
if the provided signature is small~\cite{ChenLMW014}. 
In contrast to semantic modules, \emph{deductive modules} are decidable, and focus only on entailments in $\Sigma$
that can be expressed in the DL under consideration.
Practical methods to compute them are presented in~\cite{DBLP:conf/ijcai/KoopmannC20,hui}. However, while
those modules is often half the size of the locality-based modules, for the more expressive DL \ALC, those
methods are time-consuming, and can also not always guarantee minimality.
An approximation of modules are \emph{excerpts}, whose size is bounded by the
user, but which may not preserve
%
all entailments in the given signature~\cite{zoomingin}.
%
%

Since modules are always subsets of the original ontology, they may
use names outside of the given signature.
\emph{Uniform interpolants} (UIs), on the other hand, \emph{only use names from the
provided signature $\Sigma$}, and are thus usually not syntactical subsets of the input ontology.
This makes them useful also for applications other than ontology reuse, such as
for logical difference~\cite{FORGETTING_LOGICAL_DIFFERENCE},
abduction~\cite{FORGETTING_ABDUCTION}, information
hiding~\cite{FORGETTING_PRIVACY}, and proof generation~\cite{PROOFS_FORGETTING}.
The strict
requirement on the signature means that UIs may not always
exist, and, in case of 
\ALC, can be of a size that is triple
exponential in the size of the input~\cite{DBLP:conf/ijcai/LutzW11}.
%
%
Despite this high complexity, practical
implementations for computing UIs
exist~\cite{zhao2018fame,koopmann2020lethe}. 
However, their computation times are much higher than
for module extraction and can produce very complex axioms.

Both modules and UIs can be more complex than necessary.
By dropping the syntactical requirements of those notions---being subsets of $\Omc$
and being within $\Sigma$ respectively---we may obtain ontologies that are both smaller and 
simpler, and yet still preserve all relevant entailments, which would make them better suited 
for ontology reuse and analysis. 
In this paper, we present a method to compute such \emph{general modules}, 
which are indeed often smaller and nicer than the corresponding classical modules and UIs.
Our method can indeed also compute deductive modules and UIs, and does so in significantly shorter time than the state-of-the-art.
%
While general modules have been
investigated for the lightweight DLs \EL and
\ELH~\cite{DBLP:conf/semweb/NikitinaG12,SUB_ONTOLOGIES},
to our knowledge,
this is the first time they are investigated for 
\ALC.

The main steps of our approach are shown in Figure~\ref{schema:ALC}.
Our method essentially works by performing uniform interpolation on a normalized version of the input ontology (Section~\ref{sec:normalize}). During normalization, so-called \emph{definer names} are introduced, which are eliminated in the final step.
This idea is inspired by the method for
uniform interpolation in~\cite{koopmann2020lethe}. However, different
from this approach, we put fewer constraints on the normal form and do not
allow the introduction of definers after normalization, 
which
changes the mechanism of uniform interpolation. As a result,
our definer elimination step may reintroduce names eliminated during uniform
interpolation, which is not a problem for the computation of general modules. In
contrast, eliminating definers as done in~\cite{koopmann2020lethe} can cause an
exponential
blowup, and introduce concepts with the non-standard \emph{greatest fixpoint}
constructor~\cite{calvanese2003expressive}.

A particular challenge in uniform interpolation is the elimination of role
names, for which existing approaches either rely on expensive calls
to an external reasoner~\cite{DBLP:conf/aaai/ZhaoASFSJK19,koopmann2020lethe} or
avoid the problem partially by introducing universal
roles~\cite{DBLP:conf/ijcai/ZhaoS17,%
DBLP:conf/ijcai/KoopmannC20}, leading to
results outside \ALC. A major contribution of this paper is a technique 
called~\emph{role isolation}, which allows to eliminate roles more efficiently,
and explains our short computation times
(Section~\ref{sec:RoleFGM}).

Our evaluation
shows that all our methods, including the one for uniform interpolation, can
compete with the run times of locality-based module extraction, while at the
same time resulting in subtantially smaller ontologies (Section~\ref{sec:eval}).

Our main contributions are: 1) the first method dedicated to general modules in
\ALC, 
2) a formal analysis of some
properties of the general modules we compute, 3) new methods for
module extraction and uniform interpolation that  significantly improve the
state-of-the-art, 4) an evaluation on real-world ontologies indicating
the efficiency of our technique. 

%% file: fig-tikz.tex
\begin{figure*}[!h]
\centering
\small
\scalebox{0.93}{
\begin{tikzpicture}
[
    block/.style ={rectangle,
    draw=black, thick, fill=white!20, text width=4.8em,align=center, rounded corners, minimum height=1.5em
    },sibling distance=6.5mm,level distance=25mm
]
\node[block,rectangle,text width=5em,rounded corners=0]{Ontology $\mathcal{O}$ Signature $\Sigma$}
 [edge from parent,grow=right, ->]
   child {node[block] {Ontology Normalization} 
    child {node[block] {Role Isolation \\$\textit{RI}_\Sigma$}
     child {node(c1)[block,minimum width=2em] {Role Forgetting $\rolE_\Sigma$
     } 
        child {node[block] {Concept Forgetting}
            child {node[block,text width=6.2em] {Definer Substitution/ Forgetting}
                edge from parent [level distance=35mm]
                child{node[block,text width=8.2em,rounded corners=0] {Uniform  Interpolant}}
                child{node[block,text width=8.2em,rounded corners=0] {Deductive Module}}
                child{node[block,text width=8.2em,rounded corners=0] {General Module}}
                }
            }
         }
     }
 };
    \end{tikzpicture}
    }
    \caption{Overview of our unified method for computing general modules, deductive modules, and uniform interpolants}
   \label{schema:ALC}
\end{figure*}
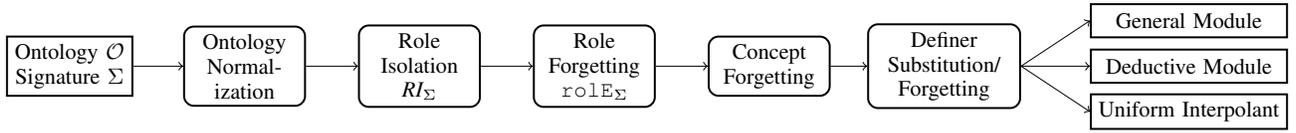

%% file: table-running-example.tex
\begin{table*}[htp]
    \centering
   \small
    \setlength{\tabcolsep}{0.5pt}
    \begin{tabular}{lccccc}
    \toprule
       $\mathcal{O}$: &  $A_1\sqsubseteq \exists r.\exists s. B_1  \sqcup
\exists r. B_2 $ & $ B_1\sqcap B_3\sqsubseteq \perp$ & $       A_2\sqsubseteq
A_3\sqcup \forall s. B_3$ & $
         B_4\sqsubseteq A_4$ & $
         B_2\sqsubseteq B_4$\\
 
         \rowcolor{gray!10}$cl(\mathcal{O})$: &
         $\neg A_1\sqcup \exists r. D_1 \sqcup \exists r. D_3$ &
         $ \neg D_1\sqcup \exists s. D_2$ &
         $ \neg D_2\sqcup B_1$ &
         $ \neg D_3\sqcup B_2$ &
         $ \neg A_2\sqcup A_3\sqcup \forall s. D_4$ \\ 
          \rowcolor{gray!10}&
         $\neg D_4\sqcup B_3$ &
         $  \neg B_1\sqcup \neg B_3$ & $ \neg B_2\sqcup B_4$ & $
          \neg B_4\sqcup A_4$&\\
 

          $\OsplitForm$: & $\neg A_1\sqcup \exists r. D_1 \sqcup \exists r. D_3$
&
          $ \neg D_1\sqcup \exists s. D_2$ &
          $
         \neg D_3\sqcup B_2$ &
          $ A_2\sqcup A_3\sqcup \forall s. D_4$ &
          $
         \neg B_1\sqcup \neg B_3$ \\

         & $\neg B_2\sqcup B_4$
         & $  \neg B_4\sqcup A_4$ &
          $  \neg D_2\sqcup \neg D_4$\\
    
          \rowcolor{gray!10}$\rolE_\Sigma(\OsplitForm)$: & $\neg A_1\sqcup \exists r. D_1
\sqcup \exists r. D_3$ &
          $  \neg D_3\sqcup B_2$ &
          $
       \neg B_1\sqcup \neg B_3$ &
          $ \neg B_2\sqcup B_4$ &
          $ \neg B_4\sqcup A_4$ \\
        \rowcolor{gray!10}& $\neg D_2\sqcup \neg D_4$ &
          $ \neg D_1\sqcup \neg A_2\sqcup A_3$&&&\\
    
       $\conE_{\Sigma}(\rolE_\Sigma(\OsplitForm)):$ & $\neg
A_1\sqcup \exists r. D_1 \sqcup \exists r. D_3$ &
          $ \neg D_1\sqcup \neg A_2\sqcup A_3$ &
          $ \neg D_3\sqcup A_4$\\
    
       \rowcolor{gray!10} $\gm_\Sigma(\mathcal{O})$: & $A_1\sqsubseteq\exists r.\exists
s.B_1\sqcup\exists r.B_2$ &
          $ A_2\sqcap\exists s.B_1\sqsubseteq A_3$ &
          $ B_2\sqsubseteq A_4$&&\\
   
       $\gm^*_\Sigma(\mathcal{O})$: & $A_1\sqsubseteq\exists r.(\neg A_2\sqcup
A_3)\sqcup\exists r.A_4$ &\\
     \bottomrule
    \end{tabular}
      \caption{Ontologies generated throughout the running example.}
    \label{tab:all_data}
\end{table*}

%% file: preliminaries_ijcai.tex
\section{Preliminaries}\label{sec:prelimGM}

We recall the DL \ALC~\cite{DESCRIPTION_LOGICS}.
Let $\NC$=$\{A,B,\ldots\}$ and $\NR = \{r, s,  \ldots\}$
be pair-wise disjoint, countably infinite sets of
\emph{concept}
and \emph{role names}, respectively.
A \emph{signature} $\Sigma\subseteq \NC\cup \NR$ is a set
of concept and role names.
\emph{Concepts} $C$
are built according to the following
grammar rules.
\begin{align}
C &::= \top ~| ~A~ | ~\neg C ~|~ C \sqcap C~ |~ C \sqcup C~ | ~\exists r. C
~|~ \forall r. C 
\end{align}
For simplicity, we identify concepts of the form $\neg\neg C$ with $C$.
In this paper, an \emph{ontology}~$\Omc$ is a finite set of
\emph{axioms} of the form $C\sqsubseteq D$, $C$ and $D$ being concepts.
We denote by $\sig(\mathcal{O})$/$\sig(C)$ the
set
of concept and role names occurring in $\mathcal{O}$/$C$, and
we use $\sigC(*)$/$\sigR(*)$ to refer to the concept/role names in $\sig(*)$.
For a signature $\Sigma$, a \emph{$\Sigma$-axiom} is an axiom $\alpha$ s.t.
$\sig(\alpha)\subseteq\Sigma$.

An \emph{interpretation} $\mathcal{I} {=} (\Delta^\mathcal{I}, \
\cdot^\mathcal{I})$ 
consists of a non-empty set~$\Delta^\mathcal{I}$ and a function~$\cdot^\Imc$
mapping each
$A\in\textsf{N}_C$ to $A^\mathcal{I}\subseteq \Delta^\mathcal{I}$ and each $r
\in \textsf{N}_R$ to
$r^\mathcal{I}\subseteq\Delta^\mathcal{I}{\times}\Delta^\mathcal{I}$. 
The interpretation function $\cdot^\Imc$ is extended to concepts as follows:
\begin{gather*}
 \top^\mathcal{I} =\Delta^\mathcal{I},\quad
 (\neg C)^\mathcal{I} = \Delta^\mathcal{I}\setminus  C^\mathcal{I},\\
 (C\sqcap D)^\mathcal{I} = C^\mathcal{I}\cap D^\mathcal{I},\quad
 (C\sqcup D)^\mathcal{I} = C^\mathcal{I}\cup D^\mathcal{I},\\
 (\exists r.C)^\mathcal{I} = \left\{a\in \Delta^\mathcal{I} \mid \exists b\in
C^\mathcal{I}: (a,b)\in r^\mathcal{I}\right\},\\
 (\forall r. C)^\mathcal{I} = \left\{a\in \Delta^\mathcal{I} \mid \forall b: (a,
b)\in r^\mathcal{I}\ra b\in C^\mathcal{I}\right\}.
\end{gather*}
An axiom $C\sqsubseteq D$ is \emph{satisfied} by an interpretation $\Imc$
($\Imc\models \text{$C\sqsubseteq D$}$) if $C^\Imc\subseteq D^\Imc$.
$\mathcal{I}$ is a \emph{model} of an ontology $\mathcal{O}$ ($\Imc\models\Omc$)
if $\Imc$ satisfies every axiom in $\Omc$. $\Omc$
\emph{entails} an axiom
$\alpha$ if $\Imc\models\alpha$ for every model $\Imc$ of $\Omc$. If $\alpha$
holds in every interpretation, we write $\models\alpha$ and call $\alpha$
an \emph{tautology}.
\newcommand{\length}[1]{\lvert #1\rvert}

The \emph{length $\length{*}$ of concepts and axioms} is defined inductively by
$\length{\top}=\length{A} = 1$, where $A\in\NC$,
$\length{C\sqcup D} = \length{C\sqcap D}=$ \mbox{$\length{C\sqsubseteq
D}$} $=\length{C}+\length{D}$,
$\length{\forall r.C} = \length{\exists r. C} =$  \mbox{$\length{C}+1$},
and
$\length{\neg C} = \length{C}$. Then, the \emph{length of an ontology},
denoted $\longont$,  is defined by
$\longont {=} \sum_{\alpha\in\mathcal{O}} \length{\alpha}$.
A central notion for us is (deductive) inseparability~\cite{Konev-MODULES,DBLP:conf/ijcai/KoopmannC20}. Given
two ontologies $\Omc_1$ and $\Omc_2$ and a signature $\Sigma$, $\Omc_1$ and
$\Omc_2$ are \emph{$\Sigma$-inseparable}, in symbols
$\Omc_1\equiv_\Sigma\Omc_2$, if for every $\Sigma$-axiom $\alpha$,
$\Omc_1\models\alpha$ iff $\Omc_2\models\alpha$.
In this paper, we are concerned with the computation of~\ourMod s, defined in
the following.
\begin{definition}[\OurMod]
Given an ontology $\mathcal{O}$ and a signature $\Sigma$, an
ontology $\mathcal{M}$ is a \emph{\ourMod}\  for $\mathcal{O}$ and
$\Sigma$ iff  (i)~$\mathcal{O}\equiv_\Sigma\mathcal{M}$ and
(ii)~$\mathcal{O}\models \mathcal{M}$.
\end{definition}
Every ontology is always a general module of itself, but we are interested in computing ones that are small in length and low in 
complexity.
%
Two extreme cases of \ourMod s are uniform interpolants and deductive modules.
\begin{definition}[Uniform interpolant \& deductive module]
Let $\mathcal{O}$ be an ontology, $\Sigma$ a signature, and $\mathcal{M}$
a \ourMod\  for $\mathcal{O}$ and $\Sigma$. 
Then, (i) $\Mmc$ is a \emph{uniform interpolant} for  $\mathcal{O}$ and $\Sigma$ if
$\sig(\mathcal{M})\subseteq \Sigma$, and (ii) $\Mmc$ is a \emph{deductive module} for  $\mathcal{O}$ and $\Sigma$ if
$\mathcal{M}\subseteq \mathcal{O}$.
\end{definition}

%% file: ontology_normalization.tex
\section{Ontology Normalization}\label{sec:normalize}
Our method performs forgetting on a normalized view of the ontology, which is
obtained via the introduction of fresh names 
as in~\cite{DBLP:phd/ethos/Koopmann15}. 
An ontology
$\mathcal{O}$ is in \emph{normal form} if every axiom 
is of the following form:
$$\top \sqsubseteq L_1\sqcup\ldots\sqcup L_n \qquad L_i::= A\mid \neg A\mid
\quant r. A,$$
where $A\in \textsf{N}_C$, and $\quant\in\{\forall, \exists\}$. 
We call the disjuncts $L_i$ \emph{literals}.
For simplicity, we omit the \enquote{$\top \sqsubseteq$} on
the left-hand side of normalized axioms, which are  regarded as
\emph{sets}, in order to avoid  dealing with duplicated literals and  order.
As an example, the  axiom $A_2\sqsubseteq A_3\sqcup \forall s. B_3$ is equivalent to  $\neg A_2\sqcup A_3\sqcup \forall s. B_3$ in normal form.


We assume a function $\cl$ that normalizes $\Omc$ usind standard transformations  
(see for example~\cite{DBLP:phd/ethos/Koopmann15}).  
In particular, $\cl$ replaces concepts~$C$ occurring under role restrictions $\quant r.C$ by fresh names $D$ taken from a
set $\ND\subseteq\NC\setminus\sigC(\Omc)$ of \emph{definers}. We use $D$, $D'$, $D_1$, $D_2$, $\ldots$ to denote definers. 
%
For each introduced definer $D$, we remember the concept $C_D$ that was replaced by it. 
We assume that distinct  occurrences of the same concept are replaced by distinct definers.
Thus, in the resulting normalization of $\Omc$ denoted $\cl(\Omc)$,
every literal $\quant r.D$ satisfies $D\in\ND$, and 
for every $D\in\ND$, $\cl(\Omc)$ contains at most one literal of the form
$\quant r.D$. 
Every definer $D$ has to occur only in literals of the form
$\neg D$ or $\quant r.D$, that is, positive literals of the form $D$ are not allowed.
Obviously, we require $\cl(\Omc)\equiv_{\sig(\Omc)}\Omc$. 



\begin{example}
\label{first_example}
Let $\mathcal{O}$ be the ontology defined in the first row of
Table~\ref{tab:all_data}. By normalizing $\Omc$, we obtain the set
$\cl(\mathcal{O})$ shown in the second row of Table~\ref{tab:all_data}. The
definers $D_1$, $D_2$ and $D_3$ in  $\cl(\mathcal{O})$ replace the concepts
$C_{D_1}=\exists s. B_1$, $C_{D_2}= B_1$ and $C_{D_3}=B_2$, respectively.

\end{example}
For a fixed normalization, we define a partial order $\preceq_d$ over all
introduced definers, which is
defined as the smallest reflexive-transitive relation over
$\ND$ s.t.
\begin{itemize}
    \item $D'\preceq_d D$ if 
    $\neg D\sqcup C\in
\cl(\Omc)$ and $D'\in\sig(C)$.
\end{itemize}
Intuitively, $D'\preceq_d D$ whenever $C_{D'}$ is contained in $C_{D}$.
In Example \ref{first_example}, we have $D_2\preceq_d D_1$, since $\neg
D_1\sqcup \exists s. D_2\in \cl(\Omc)$.
Our normalization ensures that $\preceq_d$ is acyclic.

 In the following, we 
 assume that the ontology $\mathcal{O}$
and the signature $\Sigma$ 
do not contain
definers, unless stated otherwise.

%% file: role_elimination.tex
\section{Role Forgetting}\label{sec:RoleFGM}
An ontology $\Mmc$ is called a \emph{role forgetting} for $\mathcal{O}$ and
$\Sigma$ iff $\Mmc$ is a uniform interpolant for $\Omc$ and
$\Sigma' = \Sigma\cup\sigC(\Omc)$.
Existing methods to compute role forgetting either rely on an external reasoner
\cite{DBLP:conf/aaai/ZhaoASFSJK19,koopmann2020lethe} or 
use the \emph{universal role}~$\nabla$
\cite{DBLP:conf/ijcai/ZhaoS17,%
DBLP:conf/ijcai/KoopmannC20}.
The former approach can be expensive, while the latter produces axioms
outside of \ALC. 
The normalization allows us to implement a more
efficient solution within \ALC, which relies on an integrated reasoning
procedure and an additional transformation step that produces so-called
\emph{role isolated ontologies}.

\newcommand{\Rol}{\textit{Rol}}

\subsection{Role Isolated Ontologies}\label{sec:sigmaForm}
The main idea is to separate names $A\in\NC$ that occur with roles outside of
the signature, using the following notations.
\begin{align*}
 \Rol(A, \Omc) &= \{r\in\sig(\Omc)\mid \quant r. A \text{ appears in }\Omc,\
\quant\in\{\forall, \exists\}\} \\
\unbounded_{\Sigma}(\Omc)&=\{A\in\sig(\Omc)\mid Rol(A, \Omc) \not\subseteq
\Sigma\}
\end{align*}



\begin{definition}[Role-isolated ontology]\label{defi:Sigma-norm}
An ontology $\mathcal{O}$ is \emph{\sigmaForm for $\Sigma$} if
(i) $\mathcal{O}$ is in normal form, and
(ii) every axiom $\alpha \in \mathcal{O}$ is of one of the following forms:
\begin{enumerate}[label=(c\arabic*)]
 \item  $L_1\sqcup\ldots\sqcup L_n,\ L_i:=  \neg A$ with $A\in
\unbounded_{\Sigma}(\Omc)$ for all $i$;
 \item $L_1\sqcup\ldots\sqcup L_m,\ L_i:= \quant r.A\mid B \mid \neg B $ with
$r, A\in \sig(\Omc)$, $B\not\in \unbounded_{\Sigma}(\Omc)$ for all $i$.
\end{enumerate}
\end{definition}

\noindent

Thus, an axiom in a \emph{\sigmaForm} ontology falls into two disjoint
categories: either
(c1)~it contains literals built only over concepts in
$\unbounded_{\Sigma}(\Omc)$ or (c2)~it
contains role restrictions or literals built over concepts outside
$\unbounded_{\Sigma}(\Omc)$.


\begin{example}[Example \ref{first_example} cont'd]\label{exp:notSigmaForm} 
For $\Sigma=$ $\{r$, $A_1$, $A_2$, $A_3$, $A_4\}$, 
we have
$ \unbounded_{\Sigma}(\cl(\Omc)) = \{D_2, D_4\}$.
$\cl(\mathcal{O})$ is not \sigmaForm for $\Sigma$ because of $\neg D_2\sqcup
B_1$.
%
\end{example}



Given an ontology, we compute its \sigmaForm form using the following
definition.
\begin{definition}\label{def:SigmaForm}
The \emph{\sigmaForm form} $\OsplitForm$ of $\Omc$ is defined as
$\OsplitForm := \clSigmaPart\cup \mathcal{D}_{\Sigma}(\mathcal{\Omc})$,
where 
\begin{itemize}
    \item $\clSigmaPart\subseteq\cl(\Omc)$ contains all $\alpha\in
\cl(\mathcal{O})$ s.t. if $\neg D$  is a literal of $\alpha$, then
$\Rol(D', \cl(\Omc))\subseteq \Sigma$ for all $D'\in \ND$ s.t. $D\preceq_d
D'.$
\item $\deSigmaPart$ is the set of axioms $\neg D_1\sqcup \ldots\sqcup \neg D_n$
s.t.
\textbf{(i)}~$\clSigmaPart$ contains axioms of the form
$C_1\sqcup \quant_1 r.D_1$,
$C_2\sqcup \forall r. D_2$, $\ldots$,
$C_n\sqcup \forall r. D_n$, where
$r\in \NR\setminus \Sigma$, $\quant_1\in\{\forall, \exists\}$,
and \textbf{(ii)}
 $\{D_1. \ldots, D_n\}$ is a minimal set of definers s.t.
$\cl(\mathcal{O})\models D_1\sqcap \ldots\sqcap D_n\sqsubseteq \perp$.
\end{itemize}
\end{definition}
\noindent
 Intuitively, 
  if  a definer $D$ appears in $\clSigmaPart$,
  then it should not depend on definers in
$\unbounded_\Sigma(\Omc)$.
 
 \begin{example}[Example \ref{exp:notSigmaForm} cont'd]\label{exp:normalize}
 We have:
\begin{itemize}
 \item $\clSigmaPart = \cl(\Omc)\setminus\{\neg D_2\sqcup B_1, \neg
    D_4\sqcup B_3\}$ because $\Rol(D_2, \cl(\Omc))=\Rol(D_4, \cl(\Omc)) =
\{s\}\not\subseteq \Sigma$ and, 
 \item $\deSigmaPart = \{\neg D_2\sqcup \neg D_4\}.$
\end{itemize}

\end{example} 
 


\begin{restatable}{theorem}{theoSigmanormalized}\label{theo:Sigma_normalized}
$\OsplitForm$ is \sigmaForm for $\Sigma$ and we have $\mathcal{O}\equiv_{\Sigma\cup \sigC(\mathcal{O})} \OsplitForm$.
\end{restatable}

\label{sec:pre-processing}

To compute $\deSigmaPart$, we saturate $\cl(\Omc)$ using the
inference rules shown in Fig.~\ref{fig:r-Rule}, which is sufficient due to the
following lemma.
\begin{restatable}{lemma}{lemmaPreprocessing}\label{theo-preprocessing}
Let $\mathcal{S}$ be the set of axioms 
$\neg D_1\sqcup \ldots\sqcup
\neg D_n$, 
obtained by applying the
rules in Fig.~\ref{fig:r-Rule} exhaustively
on $\cl(\mathcal{O})$.
Then, for all $D_1$, $\ldots$, $D_n\in\ND$,
we have $\cl(\mathcal{O})\models  D_1\sqcap \ldots\sqcap
D_n\sqsubseteq \perp$ iff $\neg D_{i_1}\sqcup \ldots\sqcup \neg D_{i_k} \in
\mathcal{S}$ for some subset $\{i_1,\ldots, i_k\}\subseteq \{1, \ldots, n\}$.
\end{restatable}


\begin{example}[Example \ref{exp:normalize} cont'd]\label{exp:DSigma}
The axiom $\neg D_2\sqcup \neg D_4$ in $\deSigmaPart$ is
obtained by applying two  A-Rule inferences:
\begin{table}[htp]
    \centering
    \begin{tabular}{ccc}
         $\neg D_2\sqcup B_1$,& $\neg B_1\sqcup \neg B_3 $ & \\
          \specialrule{0em}{1pt}{1pt}
         \cline{1-2}
            \specialrule{0em}{1pt}{1pt}
        \multicolumn{2}{c}{$\neg D_2\sqcup \neg B_3$,}&$\neg D_4\sqcup B_3$\\
         \specialrule{0em}{1pt}{1pt}
        \hline
           \specialrule{0em}{1pt}{1pt}
         &\multicolumn{2}{l}{\hspace{4.5mm}$\neg D_2\sqcup \neg D_4$}
    \end{tabular}
\end{table}
\end{example} 




\subsection{Role Forgetting for \SigmaForm
Ontologies}\label{sec:RolE0}
    \begin{figure}
        \centering
       \begin{tcolorbox}
\begin{eqnarray*} 
          \underline{ \emph{A-Rule}}: & \cfrac{\quad C_1\sqcup A_1\qquad \neg A_1\sqcup
C_2\quad}{C_1\sqcup C_2}\\[5pt]
  \underline{ \emph{r-Rule}}: &\cfrac{C_1\sqcup \exists r. D_1,\  \bigcup_{j=2}^{n}\{ C_j\sqcup \forall r. D_j\},\ K_D}{C_1\sqcup \ldots\sqcup C_n},
    \end{eqnarray*}
     where $K_D = \neg D_1\sqcup \ldots\sqcup \neg D_n$ or $\neg D_2\sqcup \ldots\sqcup \neg D_n$.
    \end{tcolorbox}
        \caption{Inference rules for computing $\deSigmaPart$}
        \label{fig:r-Rule}
    \end{figure}
If $\mathcal{O}$ is \sigmaForm for $\Sigma$, a role forgetting
for $\Omc$ and $\Sigma$ can be obtained using the \emph{r-Rule}
in Figure~\ref{fig:r-Rule}. 
Our method applies to any ontology in normal form, not necessarily normalized using \cl,
which is why now
%
%
the concept names $D_1,\ldots, D_n$ in the r-Rule can include also
concept names outside $\ND$.
\begin{definition}\label{defi:rolE}
$\texttt{rolE}_\Sigma(\mathcal{O})$ is the ontology
obtained as follows:
\begin{enumerate}
    \item apply the \emph{r-Rule} exhaustively for each $r\in
\sigR(\mathcal{O})\setminus \Sigma$,
    \item remove all axioms containing some
$r\in\sigR(\mathcal{O})\setminus\Sigma$.
\end{enumerate}
\end{definition}
\noindent
The second step ensures that all role names in the resulting ontology
$\texttt{rolE}_\Sigma(\mathcal{O})$ are in  $\Sigma$ and
therefore, we have
$\sig(\texttt{rolE}_\Sigma(\mathcal{O}))\subseteq \Sigma\cup\sigC(\Omc)$.

\begin{example}[Example \ref{exp:DSigma} cont'd]\label{exp:rolE}
For the ontology $\OsplitForm$, Table \ref{tab:all_data} (fourth row)  shows
$\texttt{rolE}_\Sigma(\OsplitForm)$ which is obtained through the following
two steps:
\begin{enumerate}
\item[1.] The new axiom  $\neg D_1\sqcup \neg A_2\sqcup A_3$ is generated by the
r-Rule inference:
  $$ \frac{\neg D_1\sqcup \exists s. D_2,\ \neg A_2\sqcup A_3\sqcup \forall s. D_4,\ \neg D_2\sqcup \neg D_4}{\neg D_1\sqcup \neg A_2\sqcup A_3}$$    
  
\item[2.] The two axioms $\neg D_1\sqcup \exists s. D_2,\ \neg A_2\sqcup A_3\sqcup \forall s. D_4$  are removed because they contain $s\in \sig(\OsplitForm)\setminus \Sigma$.
        \end{enumerate}
\end{example}
\begin{restatable}{theorem}{theorolE}\label{theo:rolE}
If $\mathcal{O}$ is \sigmaForm for $\Sigma$, then
$\texttt{rolE}_\Sigma(\mathcal{O})$ is a role-forgetting for $\Omc$ and $\Sigma$.
\end{restatable}





\hideThisPart{
\subsection{Role Forgetting for General Ontologies}\label{sec:RolE}
For an arbitrary ontology $\Omc$, by Theorem~\ref{theo:rolE}
and~\ref{theo:Sigma_normalized}, we can deduce that
$\mathcal{O}\equiv_{\Sigma\cup \sigC(\mathcal{O})}
\texttt{rolE}_\Sigma(\OsplitForm)$.
However, $\texttt{rolE}_\Sigma(\OsplitForm)$ may contain definers and thus not
be a role forgetting for $\Omc$ and $\Sigma$.
For instance, in Example~\ref{exp:rolE}, we can observe that
$\texttt{rolE}_\Sigma(\OsplitForm)$ contains the definers $D_1$, $D_2$, $D_3$,
$D_4$.
Actually, those definers $D$ can be eliminated using their corresponding
concepts $C_D$ as follows.

\begin{restatable}{theorem}{theorolEALC}\label{theo:rolEALC}
Let $\mathcal{R}_\Sigma(\Omc)$ be the ontology obtained by replacing each definer
$D$ in $\texttt{rolE}_\Sigma(\OsplitForm)$ by $C_D$. 
Then, we have \textbf{(i)} $\mathcal{O}\equiv_{\Sigma\cup \sigC(\mathcal{O})}
\mathcal{R}_\Sigma(\Omc)$, and  \textbf{(ii)}
if $\sig_R(C_D)\subseteq \Sigma$ for all
$D\in\sig(\texttt{rolE}_\Sigma(\OsplitForm))$, then $\mathcal{R}_\Sigma(\Omc)$
is a role forgetting for $\Omc$ and $\Sigma$.
\end{restatable}
\noindent
When $\Omc$ is normalized, $\mathcal{R}_\Sigma(\Omc)$ is always a role forgetting for $\Omc$ and  $\Sigma$ because  $\sig_R(C_D){=}\emptyset$ for every definer $D{\in} cl(\Omc)$. 
However, in the general case, $\mathcal{R}_\Sigma(\Omc)$ may not be a role forgetting. For instance, we have $\sig_R(C_{D_1}) {=} \{s\}\not\subseteq \Sigma$ in Example \ref{exp:rolE}. In this case, we can regard $\mathcal{R}_\Sigma(\Omc)$ as a role forgetting approximation.

Note that role forgetting does not always exist. 
This is the case for  $\Omc = \{A\sqsubseteq \exists s. B,\ \exists s. B
\sqsubseteq \exists r. \exists s. B,\ \exists s. B \sqsubseteq A_1\}$ and
$\Sigma = \{A, A_1, r\}$ because it is impossible to find an alternative
definition of $\exists s.B$ without using the role name $s\not\in\Sigma$.


In conclusion, we have presented how to compute a role forgetting (or its
approximation)
for each pair $\Omc$ and $\Sigma$  through three steps: 
(i) role isolation of $\Omc$ producing $\OsplitForm$,
(ii) application of the r-Rule on $\OsplitForm$,
(iii) substitution of definers~$D$ by the corresponding concepts~$C_D$.


\hui{I proposed to move Section 4.3 to appendix. And we could add the following remark following Section \ref{sec:RolE0}:
\begin{remark}
The role forgetting procedure given above is enough for computing general modules. This procedure could be generalized to the case of general ontologies using role isolation. Because of the space limitation, the details are provided in Appendix \ref{XXX}  \end{remark}
}
\patrick{Remove 4.3 entirely?}
\hui{Yes, that is the plan.}
}

%% file: extracting_general_modules.tex
\section{Computing General Modules via
$\texttt{rolE}_\Sigma$}\label{sec:computeGM}
We compute a general module from $\texttt{rolE}_\Sigma(\Omc)$ by forgetting also
the concept names 
and eliminating
all definers. 
The latter is necessary to obtain an ontology entailed by $\Omc$. Forgetting concept names is done to further simplify the ontology.

\subsection{Concept Forgetting}
We say that an  ontology $\Mmc$ is a
\emph{concept forgetting} for $\mathcal{O}$ and $\Sigma$ iff $\Mmc$ is a uniform
interpolant for $\Omc$ and the signature $\Sigma' = \Sigma\cup
\sigR(\Omc)\cup\ND$.
\patrick{I changed this: previously, $\Sigma'$ did not include the definers,
which I believe we do not eliminate in this step.}
A concept forgetting can be computed through the \emph{A-Rule} in
Figure~\ref{fig:r-Rule}. 
\begin{definition}
$\texttt{conE}_\Sigma(\mathcal{O})$ is the ontology obtained as follows:
\begin{enumerate}
\item  apply the \emph{A-Rule} exhaustively for each
$A\in\sigC(\Omc)\setminus \Sigma$, 
\item delete every axiom $\alpha$ that contains $A$ or $\neg A$, where
$A\in\NC\setminus\Sigma$ and no axiom contains $\quant r.A$ for 
$\quant\in\{\forall,\exists\}$ and $r\in\NR$.
%
\end{enumerate}
\end{definition}
\begin{example}[Example \ref{exp:rolE} cont'd]\label{exp:con_elimination}
Table \ref{tab:all_data} (the 5th row) shows the axioms in
$\texttt{conE}_{\Sigma}(\texttt{rolE}_\Sigma(\OsplitForm))$ 
obtained as follows.
\begin{enumerate}
\item $\neg D_3\sqcup B_4$, $\neg B_2\sqcup A_4$, and $\neg D_3\sqcup
A_4$ are first generated by applying the A-Rule on $B_2$ and $B_4$.
\item Axioms containing $B_i$ or $\neg B_i$, $i\in\{1,\ldots,4\}$, are
removed since $B_i\not\in\Sigma$. $\neg D_2 \sqcup \neg D_4$ is also removed
because there are no literals of the form $\quant r.D_2$ or $\quant
r.D_4$.
\end{enumerate}

\end{example} 
\noindent
The following is a consequence of~\cite[Theorem 1]{DBLP:conf/ijcai/ZhaoS17}.
\begin{theorem}\label{prop_A_rule}
If $\mathcal{O}$ is in normal form, then $\texttt{conE}_\Sigma(\mathcal{O})$ is a
concept forgetting for $\Omc$ and  $\Sigma$.
\end{theorem}

Theorems \ref{theo:Sigma_normalized}, \ref{theo:rolE} and \ref{prop_A_rule},
give us the following corollary.
\begin{corollary}\label{cor:con_equiv}
$\texttt{conE}_{\Sigma}(\texttt{rolE}_\Sigma(\OsplitForm))
\equiv_{\Sigma} \Omc$.
\end{corollary}

\subsection{Constructing the General Module}

Now, in order to obtain our general modules, we have to eliminate the definers from
$\texttt{conE}_{\Sigma}(\texttt{rolE}_\Sigma(\OsplitForm))$. To improve
the results, we delete \emph{subsumed axioms} (i.e., axioms
$\top\sqsubseteq C\sqcup D$ for which we also derived $\top\sqsubseteq
C$) and also simplify the axioms.


\begin{restatable}{theorem}{maintheo}\label{main_theo}
Let $\gm_\Sigma(\mathcal{O})$ be the ontology obtained from
$\texttt{conE}_{\Sigma}(\texttt{rolE}_\Sigma (\OsplitForm))$ by
\begin{itemize}
 \item deleting subsumed axioms,
 \item replacing each definer $D$ by $C_D$, and
 \item exhaustively applying $C_1\sqsubseteq\neg C_2\sqcup C_3$ $\Rightarrow$
 $C_1\sqcap C_2\sqsubseteq C_3$ and 
 $C_1\sqsubseteq\quant r.\neg C_2\sqcap C_3$ $\Rightarrow$ 
 $C_1\sqcap\overline{\quant}r.C_2\sqsubseteq C_3$, where ${\overline{\exists}}={\forall}$ and ${\overline{\forall}}={\exists}$.
\end{itemize}
%
Then, $\gm_\Sigma(\mathcal{O})$ is a general module for $\mathcal{O}$ and
$\Sigma$.
\end{restatable}

\begin{example}[Example \ref{exp:con_elimination} cont'd]\label{exp:general_mod}
Table \ref{tab:all_data} (the 8th row) shows the general module
$\gm_\Sigma(\mathcal{O})$, which has been obtained using
$C_{D_1}{=}\exists s. B_1$ and $C_{D_3}{=}B_2$. 
%
\end{example}
%
Eliminating definers in this way may reintroduce previously forgotten names,
which is why our general modules are in general not uniform interpolants. This
has the
advantage of avoiding the triple exponential blow-up caused by uniform
interpolation (see Section~\ref{sec:intro}).
In contrast, the size of our result is at most single exponential in the size
of the input.

\begin{restatable}{proposition}{propsizegm}\label{prop:size_gm}
For any ontology  $\Omc$  and signature  $\Sigma$, we have
$\lVert \gm_\Sigma(\mathcal{O})\rVert  \leq 2^{O(\longcl)}$.
On the other hand, there exists a family of ontologies $\Omc_n$ and signatures $\Sigma_n$ s.t.
$\lVert\Omc_n\rVert$ is polynomial in $n\geq 1$ and
$\lVert  \gm_{\Sigma_n}(\mathcal{O}_n) \rVert  = n\cdot 2^{O(\lVert\cl(\Omc_n)\rVert)}$.
\end{restatable}
We will see in Section~\ref{sec:eval} that this theoretical bound is usually not reached in practice, and usually 
general modules are much smaller than the input ontology.

For some module extraction methods, such as for locality-based
modules \cite{grau2008modular}, iterating the computation can lead to smaller
modules.
The following result shows that this is never the case for our method.

\begin{restatable}{proposition}{propgmMono}\label{prop:gmMono}
Let 
$(\mathcal{M}_i)_{i\geq 1}$ be the sequence of ontologies defined by
(i)~$\mathcal{M}_1 = \gm_\Sigma(\mathcal{O})$ and (ii) $\mathcal{M}_{i+1} =
\gm_\Sigma(\mathcal{M}_i)$ for $i\geq 1$. Then, we have 
$\mathcal{M}_i\subseteq \mathcal{M}_{i+1}\text{ for }i\geq 1.$ 
Moreover, there exists $i_0\geq 0$ s.t. $\mathcal{M}^k = \mathcal{M}^{i_0}
\text{ for all }k\geq i_0$.
\end{restatable}


This property holds thanks to the substitution step of Theorem \ref{main_theo}.
This step may reintroduce in $\gm_\Sigma(\Omc)$ concept and role names outside
of $\Sigma$.
As a result, the repeated application of $\texttt{rolE}_\Sigma$ and
$\texttt{conE}_\Sigma$ on $\gm_\Sigma(\Omc)$ can produce additional but
unnecessary axioms.
However, for ontologies in normal form, our method is \emph{stable} in the sense
that repeated applications produce the same ontology.

\begin{restatable}{proposition}{propspecialO}\label{prop_specialO}
Let $\mathcal{O}$ be an ontology in normal form and $\mathcal{M} = \gm_\Sigma(\Omc)$. Then,
$\gm_\Sigma(\mathcal{M})=\mathcal{M}$.
\end{restatable}


\begin{figure}
    \centering
   \begin{tcolorbox}
       \underline{\emph{conD-Elim}}:
   \begin{equation}\nonumber
      \frac{C_1\sqcup \quant r.D_1, \bigcup_{j=2}^{n}\{ C_j\sqcup \forall r.
D_j\}, \neg D_1\sqcup\ldots\sqcup \neg D_n}{C_1\sqcup \ldots\sqcup
C_n\sqcup\quant r. \bot}
    \end{equation}
   \begin{equation}\nonumber
 \underline{\emph{D-Prop}}:  \hspace{5mm}   \frac{C_1\sqcup \quant r.D,\quad
\bigcup_{j=2}^{n}\{ \neg D\sqcup C_j\}}{C_1\sqcup \quant r.(C_2
\sqcap\ldots\sqcap C_n)}\hspace{15mm}
    \end{equation}
    where $\bigcup_{j=2}^{n}\{ \neg D\sqcup C_j\}$ ($n\geq 1$) are all the axioms
of the form $\neg D\sqcup C$. This rule is applicable only if no 
$C_j$ contains definers. 

\end{tcolorbox}
    \caption{Rules to eliminate definers}
    \label{fig:D-Prop}
\end{figure}

\subsection{Optimizing the Result}\label{sec:elimD}
The general module $\gm_\Sigma(\Omc)$  may contain complex axioms since
definers $D$ can stand for complex concepts $C_{D}$.
To make the result more concise, we eliminate some definers before substituting
them. In particular, we use the following operations on
$\texttt{conE}_{\Sigma}(\texttt{rolE}_\Sigma(\OsplitForm))$, inspired
by~\cite{DBLP:conf/dlog/SakrS22}.
\begin{enumerate}
\item[Op1.] \textbf{Eliminating conjunctions
of definers} 
aims to eliminate
disjunctions of negative definers ($\neg D_1\sqcup \ldots\sqcup \neg
D_n$).
This is done in two steps: 
(i) Applying the \emph{conD-Elim} rule in Figure~\ref{fig:D-Prop} on
$\texttt{conE}_{\Sigma}(\texttt{rolE}_\Sigma(\OsplitForm))$, and then
(ii) deleting all axioms of the form $\neg D_1\sqcup \ldots\sqcup \neg D_n$.

\item[Op2.] \textbf{Eliminating single definers} aims to get rid of definers
$D$ that do not occur in axioms of the form $\neg D\sqcup \neg D_1\sqcup C$.
This is done in two steps: 
(i) applying the \emph{D-Prop} rule of Figure~\ref{fig:D-Prop}
exhaustively and then
(ii) deleting all axioms containing definers for which \emph{D-Prop}
has been applied.

\end{enumerate}

\begin{restatable}{theorem}{corgmOpti}\label{cor:gmOpti}
Let $\gm^*_\Sigma(\mathcal{O})$ be the ontology obtained by:
\begin{itemize}
    \item successive application of Op1 and Op2 over
$\texttt{conE}_{\Sigma}(\texttt{rolE}_\Sigma(\OsplitForm))$, followed by
    \item application of the steps described in Theorem~\ref{main_theo}.
\end{itemize}
Then, $\gm^*_\Sigma(\mathcal{O})$ is a general module for $\mathcal{O}$ and
$\Sigma$.
\end{restatable}



\begin{example}\label{exp:opti}
Assume $\Sigma = \{r, A, A_1\}$ and 
$$\mathcal{O} = \{A\sqsubseteq \forall r. \exists s. B_1, \ A_1 \sqsubseteq \forall r. \forall s.B_2, \ B_1\sqcap B_2\sqsubseteq \bot\}.$$
Then, $\texttt{conE}_{\Sigma}(\texttt{rolE}_\Sigma(\OsplitForm))$ is:
$$ \{\neg A\sqcup \forall r. D_1,\ \neg A_1 \sqcup \forall r. D_2,\ \neg D_1\sqcup \neg D_2\},$$
where $C_{D_1} = \exists s. B_1,\ C_{D_2} = \forall s.B_2$. And thus, by
replacing $D_i$ by $C_{D_i}$, we obtain $\gm_\Sigma(\Omc)=$
$$\{A\sqsubseteq \forall r. \exists s. B_1, \ A_1 \sqsubseteq \forall r. \forall
s.B_2,\ \exists s. B_1\sqcap \forall s.B_2\sqsubseteq \bot\},$$
which is actually more intricate than $\mathcal{O}$. 
We can  avoid this by applying the two optimizations described
above.

The elimination of definer conjunctions (Op1) produces
    \begin{equation}\label{firstOp}
    \{\neg A\sqcup \forall r. D_1, \neg A_1 \sqcup \forall r. D_2, \neg A\sqcup \neg A_1\sqcup \forall r. 
    \bot \}.\end{equation}
(i) The first step of Op1  applies the conD-Elim inference:
 \begin{equation}\nonumber
      \frac{\neg A\sqcup \forall r. D_1,
      \neg A_1 \sqcup \forall r. D_2, \neg D_1\sqcup \neg D_2}{\neg A\sqcup \neg
      A_1\sqcup \forall r. \bot}.
    \end{equation}
\noindent   
(ii) The second step of Op1 removes the axiom $\neg D_1\sqcup \neg D_2$.

   Then, the elimination of definers (Op2) produces
       \begin{equation}\label{SecondOp}
       \{\neg A\sqcup \forall r. \top, \neg A_1 \sqcup \forall r. \top, \neg
            A\sqcup \neg A_1\sqcup \forall r.\bot \}
    \end{equation}
   by replacing $D_1, D_2$ by $\top$ as there is no axioms with negative $D_i, i
= 1, 2$ in Equation (\ref{firstOp}).
    Note that  the first two axioms in Equation (\ref{SecondOp}) are tautologies
and thus can be ignored.

   Finally, we have $\gm^*_\Sigma(\mathcal{O}) = 
   \{A\sqcap A_1\sqsubseteq\forall r.\bot\}$.
    \end{example}


%% file: extracting_deductive_modules.tex
\section{Deductive Modules and Uniform Interpolants}\label{sec:dmAndUI}
\paragraph{Deductive modules}
Depending on the situation, users might prefer the axioms in the original
ontology $\Omc$ rather than newly introduced axioms (e.g., axioms in
$\gm_\Sigma(\Omc)$ or $\gm^*_\Sigma(\Omc)$).
For such situations, we can compute a deductive module for $\Omc$ and $\Sigma$ by
tracing back the inferences performed when computing
the general module
$\gm^*_\Sigma(\Omc)$.

Let $\Res_\Sigma(\Omc)$ be the set of all axioms generated by the computation
progress of  $\gm^*_\Sigma(\Omc)$.
Clearly, $\gm^*_\Sigma(\Omc)\subseteq \Res_\Sigma(\Omc)$. 
We iteratively construct a relation $R$ on $\Res_\Sigma(\Omc)$ during the computation $\gm^*_\Sigma(\Omc)$ as follows:
we start with $R=\emptyset$, and 
each time a new axiom $\beta$ is generated from a premise set $\{\alpha_1, \ldots, \alpha_n\}$ (e.g., if $\beta$ is obtained by applying r-Rule on
$\{\alpha_1, \ldots, \alpha_n\}$), we add to $R$ the relations $\alpha_1 R\beta$, \ldots, $\alpha_n R\beta$. 
Let $R^*$ be the smallest transitive closure of $R$. Then the deductive module is defined as follows.

\begin{restatable}{theorem}{theodeduct}\label{theo:deduct}
Let us define $\Dm_\Sigma(\Omc)$ by
\begin{equation}\nonumber
   \Dm_\Sigma(\Omc) = \{\alpha\in\Omc\mid \alpha R^* \beta \text{ for some }\beta\in
\gm^*_\Sigma(\mathcal{O})\}.
\end{equation}
Then, $\Dm_\Sigma(\Omc)$ is a deductive module for $\Omc$ and $\Sigma$.
\end{restatable}


\paragraph{Uniform interpolants}
While general modules can be a good alternative to uniform interpolants for
ontology reuse, 
uniform interpolation has applications 
that require the ontology to be fully in the selected signature, as stated in the introduction. 
If instead of substituting definers $D$ by $C_D$, we eliminate them using existing uniform interpolation tools, we can compute a uniform interpolant for the input.
%
\patrick{Maybe move these use cases to the introduction.}

When computing $\gm^*_\Sigma(\Omc)$ for an ontology $\mathcal{O}$ and 
signature $\Sigma$, if all definers have been eliminated by Op1 and Op2
from Section~\ref{sec:elimD}  (as in Example~\ref{exp:opti}), then
$\gm^*_\Sigma(\Omc)$ is indeed 
a uniform interpolant for $\mathcal{O}$ and
$\Sigma$.
Otherwise,
we compute a uniform interpolant by forgetting the
remaining definers using an existing uniform
interpolation tool such as~\Lethe or
\Fame~\cite{koopmann2020lethe,zhao2018fame}. 
As we see in
Section~\ref{sec:eval}, this allows us to compute uniform interpolants much
faster
than using 
the tool alone.

%% file: evaluation.tex
\section{Evaluation}\label{sec:eval}

To show that our general modules can serve as a better alternative for ontology reuse and analysis, 
we compared them 
with the state-of-the-art tools implementing module extraction and uniform interpolation 
for \ALC. 
We were also interested in the impact of our optimization, and the performance of our technique
for computing deductive modules and uniform interpolants.
We 
implemented a prototype called \ourM in Python 3.7.4.  
As evaluation metrics, we looked at
run time, length of computed ontologies, and length of largest axiom in the
result.
\patrick{We should also publish our experimental data, maybe on zenodo. The source code could also be 
published on github.}
All the experiments were performed on a machine with an Intel Xeon Silver 4112 2.6GHz, 64
GiB of RAM, Ubuntu 18.04, and OpenJDK 11.
\patrick{Since you are using Java in your experiments as well, you should
mention which JVM you used (there are massive differences in performance.)}
\hui{like this? openjdk 11.0.17}
\paragraph{Corpus} The ontologies used in our experiment are generated from the
OWL Reasoner Evaluation (ORE) 2015 classification track
\cite{DBLP:journals/jar/ParsiaMGGS17} by the two following steps.
First, we removed axioms outside of $\mathcal{ALC}$ from each ontology in ORE~2015.
Then, we kept the ontologies $\mathcal{O}$ for which
$\cl(\mathcal{O})$ contained between 100 and 100,000 names.
This resulted in 222 ontologies.

\paragraph{Signatures} For each ontology, we generated 50 signatures consisting
of 100 concept and role names. As in~\cite{DBLP:conf/ijcai/KoopmannC20}, we selected each
concept/role name with a probability  proportional to their occurrence
frequency in the ontology. In the following, a \emph{request} is a pair consisting of an ontology
and a signature.
%

\patrick{Wait a moment - you measured 1.13 seconds on average for your
computation of general modules. Does this not include the preprocessing time??
That would not be a fair comparison...}
\patrick{Also, it would be interesting to also know the times \emph{per
ontology}, rather than just per role. }


\paragraph{Methods}
For each
request $(\mathcal{O}, \Sigma)$, \ourM produced three different (general)
modules
$\gm_\Sigma(\Omc)$, $\gm^*_\Sigma(\Omc)$  and $\Dm_\Sigma(\Omc)$,
respectively
denoted by \GM, $\GM^*$, and \DM.
\GMLETHE denotes the uniform interpolation method described
in Section~\ref{sec:dmAndUI}, where we used \ourM for computing $\GM^*$ and then  \Lethe for definer forgetting. In the implementation, for each request, we first extracted a locality based \emph{\starM}~\cite{grau2008modular} to accelerate the computation. This is a common practice also followed by the uniform interpolation and deductive module extraction tools used in our evaluation.
Since removing subsumed axioms as mentioned in Theorems~\ref{main_theo} and \ref{cor:gmOpti} can be challenging, we set a time limit of 10s for this task.
\patrick{The description of redundancy elimination is now mentioned in the section on general modules, and can go there. But we should still mention the time out. (Would be interesting to know how often it was used in the end. We should investigate this after we submitted the paper.)}
\hui{Sorry, I don't record this number.}

We compared our methods with four different alternatives:
\textbf{(i)}~\emph{\starM}s \cite{grau2008modular} as implemented in the OWL API \cite{OWL_API};
\patrick{Version of the OWL API?}
\hui{I use a existing jar file. I don't know the version of OWL API.}
\textbf{(ii)}~\minM \cite{DBLP:conf/ijcai/KoopmannC20} that computes
\emph{minimal deductive modules} under $\ALCH^\nabla$-semantics;
\textbf{(iii)}~\Lethe~
0.6\footnote{\url{https://lat.inf.tu-dresden.de/~koopmann/LETHE/}}%
\cite{koopmann2020lethe}  and \Fame~1.0\footnote{\url{http://www.cs.man.ac.uk/~schmidt/sf-fame/}} \cite{zhao2018fame}
that compute uniform interpolants. 



\paragraph{Success rate} We say a method \emph{succeeds} on a request if it
outputs the expected results within~600s.
Table~\ref{table_success} summarizes the success rate for the methods considered.
After the \starM s, our method \ourM had the highest success rate.
\begin{table}[]
    \centering     
    \small
     \setlength{\tabcolsep}{2pt}
    \begin{tabular}{c c cc cc}
    \toprule
         \starM &  \minM & \Lethe & \Fame &\ourM& \GMLETHE\\
        \midrule
        {\color{red}100}\%&  84.34\% & 85.27\% & 91.25\% & {\color{blue}97.34}\%& 96.17\%\\ 
    \bottomrule
    \end{tabular}
    \caption{Success rate evaluation. The {\color{red}first} (resp. {\color{blue}second}) best-performing method
is highlighted in {\color{red}red} (resp. {\color{blue}blue}).}
    \label{table_success}
\end{table}
\paragraph{Module length and run time}
Because some of the methods can change the shape of axioms, the number of axioms is not a good metric for understanding the quality of general modules.
We thus chose to use ontology length as defined in 
Section~\ref{sec:prelimGM}, rather than size, for our evaluation. 
Table \ref{exp:mainCompare} shows the length and run time for the requests on which all 
methods were successful ($78.45\%$ of all requests).

We observe that $\DM$ and \GMLETHE have the best
overall performance: their results had a substantially smaller average length and were computed significantly faster than others. 
Note that the average size of results for $\DM$ was even smaller than that for $\minM$. 
The reason is that $\minM$ preserves entailments over $\ALCH^\nabla$, while we preserve only entailments over $\ALC$.
%
Therefore, the $\minM$ results may contain additional axioms compared to the $\ALC$ deductive modules.


Comparing $\GM$ and $\GM^*$ regarding length lets us conclude that the
optimization in Section~\ref{sec:elimD} is effective.
On the other hand, \minM produced results of small length but at
the cost of long computation times. \Fame and  \starM\ were quite time-efficient but  
less satisfactory in  size, especially for \Fame, whose results are often considerably
larger than for the other methods. \Lethe took more time than \Fame,
but produced more concise uniform interpolants on average.
\patrick{Since this paper is about our method and not about comparing the state of the art,
we may leave the comparison of existing methods out if space becomes an issue. }
\hui{We have space now.}

For $87.87\%$ of the requests reported in Table~\ref{exp:mainCompare},  $\GM^*$ already computed a uniform interpolant, so that \GMLETHE did not need to perform any additional computations. 
\newcommand{\tabincell}[2]{\begin{tabular}{@{}#1@{}}#2\end{tabular}}

\begin{table}[]
    \centering
      \small
    \setlength{\tabcolsep}{0.5pt}
    \begin{tabular}{c cc c}
    \toprule
    \multicolumn{2}{c}{Methods} & Resulting ontology length & Time cost\\
    \midrule
        \multicolumn{2}{c}{\minM} &   {\color{red}2,355} / 392.59 / 264 & 595.88 / 51.82   / 8.86\\
  \rowcolor{gray!10} \multicolumn{2}{c}{\starM} & 4,008 / 510.77 / 364 &  {\color{red}5.94} / {\color{red}1.03} / {\color{red}0.90}  \\
    \multicolumn{2}{c}{\Fame} & 9,446,325 / 6,661.01 / 271 &   526.28 / 3.20 / {\color{blue}1.17}\\
     \rowcolor{gray!10}    \multicolumn{2}{c}{\Lethe} & 131,886 / 609.30 / 196  &  598.20 / 49.21 / 13.57\\
 \multirow{3}*{ \ourM} & \GM &  179,999 / 2,335.05 / 195 & \multirow{3}*{{\color{blue}17.50} / {\color{blue}2.44} / 1.63}\\
   & $\texttt{gm}^*$ & 21,891 / 466.15 / {\color{blue}166} \\
    & \texttt{dm} &{\color{blue}2,789} /  {\color{blue}366.36} / 249\\
    \rowcolor{gray!10}   \multicolumn{2}{c}{\GMLETHE}& 21,891 / {\color{red}364.10} /{\color{red}162} &  513.15 / 3.08 / 1.68\\
  \bottomrule
    \end{tabular}
    \caption{Comparison of different methods (max. / avg. / med.).}
    \label{exp:mainCompare}
\end{table}

Figure \ref{figure:eval} provides a detailed comparison of \minM, $\GM^*$ and
\GMLETHE. $\GM^*$ was often faster but produced larger results. 
In contrast, \GMLETHE produced more concise results at the cost of longer computation time. 
While \minM avoided large modules, it was generally much slower than our methods. 
  \begin{figure}[]
     \centering
     \includegraphics[width = 7cm]{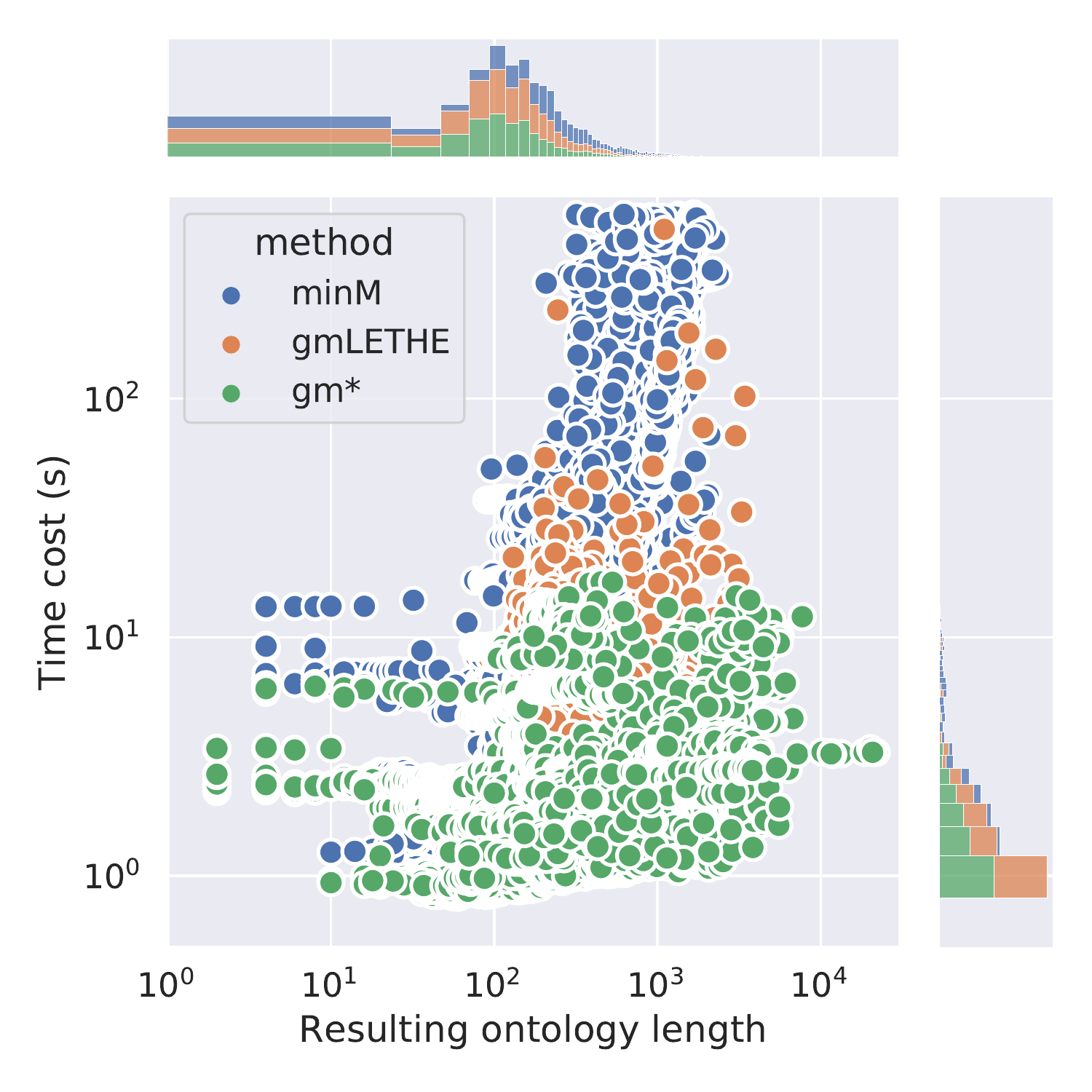}
     \caption{
     Comparison of   \texttt{minM}, \texttt{gmLETHE}, and $\texttt{gm}^*$. 
     }
     \patrick{What I do not understand is why blue does not occur at all on the bar charts on the right. Shouldn't the amount of each color be the same for all methods?}
     \hui{It is here, just being too small. You can find it by zooming in.}
     \hui{I comment out the logarithmic scale because it is shown in the axis of the figure}
     \label{figure:eval}
 \end{figure}

Table \ref{exp:mainOurs} summarizes the results concerning all requests for which \ourM (resp. \GMLETHE) was successful. We see that the results of $\DM$ had a small average size. However, as for $\GM^*$ and $\GMLETHE$, the median size of results was much smaller, which suggests that  $\GM^*$ and $\GMLETHE$ perform better over relatively simple cases. 
\begin{table}[]
    \centering
    \small
      \setlength{\tabcolsep}{2.5pt}
    \begin{tabular}{c cc c}
    \toprule
    \multicolumn{2}{c}{Methods} & Resulting ontology length & Time cost\\
\midrule
 \multirow{3}*{ \ourM} & \GM &  17,335,040 / 35,008.2 / 310 & \multirow{3}*{585.97 / 4.89 / 1.75}\\
   & $\texttt{gm}^*$ & 2,318,878 / 2,978.77 / 214 \\
    & \texttt{dm} & 18,218 / 638.74 / 309\\
    \rowcolor{gray!10} \multicolumn{2}{c}{\GMLETHE}& 353,107 / 1,006.34 /192 &  579.70 / 7.56 / 2.02\\
  \bottomrule
    \end{tabular}
    \caption{\ourM\  and \GMLETHE : Summary of results for all their own successful experiments (max. / avg. / med.).}
    \label{exp:mainOurs}
\end{table}
\paragraph{Uniform interpolants} 
For $80.23\%$ of requests where \ourM was successful, 
$\gm^*_\Sigma(\Omc)$ was already a uniform interpolants. 
In the other cases, 
 the success rate for \GMLETHE was $93.96\%$.  
In the cases where \Lethe failed, the success rate for \GMLETHE was $36.23\%$. 
\patrick{This sentence is hard to parse: do you mean: take all requests for which 
$\GM^*$ produced no uniform interpolant and \Lethe timed out. Of those, \GMLETHE was successful 
in $25\%$ of cases? The following numbers might be easier to grasp:
1. in the cases where $\GM^*$ did not produce a uniform interpolant, 
what was the success rate for \GMLETHE?
2. in the cases where \Lethe timed out, what was the success rate for \GMLETHE?
}

The comparison of \Lethe with \GMLETHE in
Figure~\ref{figure:gmlethe} shows that \GMLETHE was significantly faster than
\Lethe in most of the cases.
  \begin{figure}
     \centering
     \includegraphics[width =5.5cm]{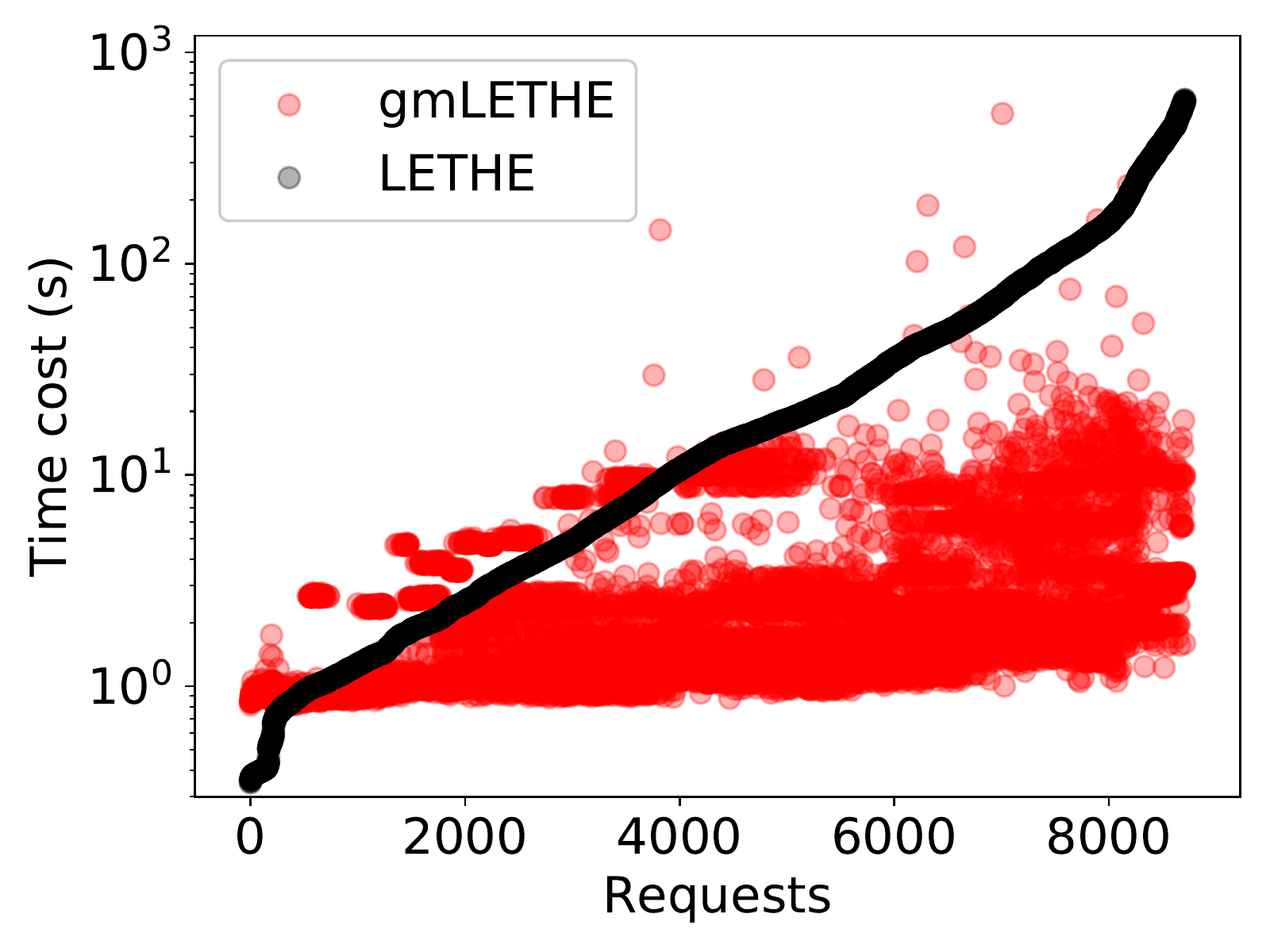}
     \caption{Run time comparison of \Lethe and \GMLETHE.}
     \label{figure:gmlethe}
 \end{figure}

\paragraph{Axiom size}
A potential shortcoming of general modules compared to classical modules is that they 
could contain axioms that are more complex than those of the input, and thus 
be harder to handle by human end-users. 
For the requests reported in Table~\ref{exp:mainCompare}, the largest axiom in the 
output of \minM had length 352, while for $\GM^*$, it had length 5,815, and for $\GM$, even only 56. 
In contrast, for the uniform interpolants computed by \Lethe and \Fame, the situation was much worse: 
here, the largest axiom had a length of 26,840 and 130,700, respectively, which is clearly beyond what can be understood by a human end-user. 
Besides these extreme cases, we can also observe differences wrt. the median values:
for $\GM^*$ the longest axiom had a median length of 3, 
which is even lower than the corresponding value for 
\minM (5),  and, as expected, lower than for \Lethe (4) and \Fame (6).
This indicates that, in most cases,
general modules computed by $\GM^*$ are simple than for the other tools.
\patrick{At this point, one would really like to know the situation for \GMLETHE. Do we know the numbers there?}

\patrick{Subsection on preprocessing: rerun }

%% file: conclusion.tex
\section{Conclusion}\label{sec:concluGM}
We presented new methods for 
computing 
general modules for \ALC ontologies, which can also be  used for computing deductive modules and uniform interpolants. Due to its higher syntactical flexibility, our general modules are often smaller and less complex than both classical modules and uniform interpolants computed with the state-of-the-art, which makes them particularly useful for applications such as ontology reuse and ontology analysis.
Our method is based on a new role isolation process that enables efficient role forgetting and an easy definer elimination. 
The experiments on real-world ontologies validate the efficiency of our proposal and the quality of the computed general modules. 
In the future, we want to optimize the concept elimination step to obtain more concise general modules. Also, we would like to investigate how to generalize our ideas to more expressive DLs.

%% file: appendix.tex
\newpage
\onecolumn
\patrick{Remark: I deactivate double columns to make it easier to read this
part. I don't know whether there are any requirements for the appendix, but we
can always switch back to double column later.
}

\section{Proofs of Results from the Paper}
The order in which we prove the results differs slightly from the order in which they appear in the main text. In particular,
since the proof  of Proposition~\ref{prop:gmMono} depends on some arguments used for Proposition~\ref{prop_specialO}, we show it after the proof for Proposition~\ref{prop_specialO}.
\patrick{Check: still called corollary?}

\input{appendix-computing-RI}

\input{appendix-role-elimination}

\subsection{Proof of Theorem \ref{main_theo}}\label{proof_main_theo}
Let $\Sigma$ be a signature. For $X$ being a role name or a concept, we define 
$\texttt{copy}_\Sigma(X)$ by structural induction.

\begin{enumerate}
    \item $\texttt{copy}_\Sigma(A)= A$, if $A\in (\NC\cap\Sigma)\cup \ND$;
    \item $\texttt{copy}_\Sigma(r)= r$, if $r\in \NR\cap\Sigma$;
    \item $\texttt{copy}_\Sigma(B) = \overline{B}$ if $B\in\sigC(\Omc)\setminus\Sigma$, where $\overline{B}$ is fresh;
    \item $\texttt{copy}_\Sigma(r) = \overline{r}$ if $r\in\sigR(\Omc)\setminus\Sigma$, where $\overline{r}$ is fresh;
    \item $\texttt{copy}_\Sigma(C_1\sqcup C_2) = \texttt{copy}_\Sigma(C_1) \sqcup \texttt{copy}_\Sigma(C_2)$;
    \item $\texttt{copy}_\Sigma(C_1\sqcap C_2) = \texttt{copy}_\Sigma(C_1) \sqcap \texttt{copy}_\Sigma(C_2)$;
    \item $\texttt{copy}_\Sigma(\quant r. C) = \quant{\texttt{copy}_\Sigma(r)}.\texttt{copy}_\Sigma(C)$, \item $\texttt{copy}_\Sigma(\neg C) =\neg \texttt{copy}_\Sigma(C)$.
\end{enumerate}

We further define $\cl^{ex}(\mathcal{O})  = \OsplitForm\cup \cl(\mathcal{O}_D)$, where  
\begin{equation}\label{O-D}
\mathcal{O}_D=\{ D\equiv \texttt{copy}_\Sigma(C_D)\mid D\in \sig( \texttt{conE}_{\Sigma}(\texttt{rolE}_{\Sigma}(\OsplitForm)))\cap \ND\}.
\end{equation}

To prove Theorem \ref{main_theo}, we need the following two lemmas.
\begin{lemma}\label{claim_ex_rolE}
We have $\mathcal{O}\equiv_{\Sigma\cup \sigC(\Omc)}  \cl^{ex}(\mathcal{O}) $.
\end{lemma}
\begin{proof}
Since $ \OsplitForm\subseteq \cl^{ex}(\mathcal{O}) $, and $\mathcal{O}\equiv_{\Sigma\cup \sigC(\mathcal{O})} \OsplitForm$ (by Theorem~\ref{theo:Sigma_normalized}), we have for every axiom $\alpha$ s.t. $\sig(\alpha)\subseteq\Sigma\cup\sigC(\Omc)$ and $\Omc\models\alpha$, also $\cl^{ex}(\Omc)\models\alpha$. 
We thus only need to show the other direction.

Let $\alpha$ be s.t. $\sig(\alpha)\subseteq\Sigma\cup\sigC(\Omc)$, and assume $\cl^{ex}(\Omc)\models\alpha$ but  $\Omc\not\models\alpha$. Then, also $\cl(\Omc)\not\models\alpha$, and there is a witnessing model $\Imc$ of $\cl(\Omc)$ s.t. $\Imc\not\models\alpha$. Based on $\Imc$, we construct a model $\Imc^{ex}$ of $\cl^{ex}(\Omc)$ and also show that  $\Imc^{ex}\not\models\alpha$, and thus $\cl^{ex}(\Omc)\not\models\alpha$. A contradiction!

$\Imc^{ex}$ is defined as follows.
\begin{enumerate}
    \item $A^{\mathcal{I}^{ex}}= A^\mathcal{I}, r^{\mathcal{I}^{ex}}= r^\mathcal{I}$ for all $A, r\in \sig(\mathcal{O})\cap \sig(cl^{ex}(\Omc))$;
    \item $\overline{r}^{\mathcal{I}^{ex}}= r^{\mathcal{I}^{ex}}$, $\overline{B}^{\mathcal{I}^{ex}}= B^{\mathcal{I}^{ex}}$ for all introduced role names $\overline{r}$ and introduced concept names $\overline{B}$.
    \item $D^{\mathcal{I}^{ex}}= (C_D)^{\mathcal{I}^{ex}}$ for every definer $D\in \sig(\cl^{ex}(\mathcal{O}))\cap\ND$.
\end{enumerate}
Since $\mathcal{I}$ is a model of $\Omc$, by the item 3 above and the definition of $C_D$, we know  $\Imc^{ex}$ is compatible with all axioms in $\OsplitForm$ and $\mathcal{O}_D$. Therefore, $\Imc^{ex}$ is compatible with all axioms in $cl(\Omc_D)$  and thus a model of $cl^{ex}(\Omc_D)$. Moreover, 
because 
$A^{\mathcal{I}^{ex}}= A^\mathcal{I}$ and $r^{\mathcal{I}^{ex}}= r^\mathcal{I}$ for all $A, r\in\Sigma\cup \sig(\mathcal{O})$, we have $\Imc^{ex}\models\cl^{ex}(\Omc)$ and $\Imc^{ex}\not\models\alpha$. A contradiction. \patrick{Still needs to show that $\Imc^{ex}$ is a model of $\Omc_D$.}
\hui{I add a sentence}
%
\end{proof}

\begin{lemma}\label{claim_gm}
$\cl^{ex}(\Omc) \equiv_\Sigma  \gm_\Sigma(\mathcal{O})$.
\end{lemma}

\begin{proof}
Assume 
$\Sigma^{ex} =
\Sigma\cup \sig( cl(\Omc_D))$.
%
Then $\sigR(cl(\Omc_D))=\sigR(\Omc_D)\subseteq \Sigma^{ex} $. Since $\OsplitForm$ is \sigmaForm for $\Sigma$, we have $\cl^{ex}(\mathcal{O})=\OsplitForm\cup \cl(\mathcal{O}_D)$ is \sigmaForm for $\Sigma^{ex}$. 
Moreover, we observe the following:
\begin{enumerate}
    \item $\texttt{rolE}_{\Sigma^{ex}}(\cl^{ex}(\mathcal{O})) =\texttt{rolE}_{\Sigma}(\OsplitForm)\cup cl(\mathcal{O}_D).$ 
    
    This is because $\sig_R(cl(\Omc_D))\subseteq  \Sigma^{ex}$ and  $cl(\Omc_D)$ does not contain axioms of the form $\neg D_1\sqcup \ldots\sqcup \neg D_n$, then the r-Rule is only applicable on  $\OsplitForm$.
   
   \item
$\texttt{conE}_{\Sigma^{ex}}(\texttt{rolE}_{\Sigma^{ex}}(\cl^{ex}(\mathcal{O})
)) = \texttt{conE}_{\Sigma}(\texttt{rolE}_{\Sigma}(\OsplitForm))\cup
cl(\mathcal{O}_D).$
    
    This is because $\sig_C(cl(\Omc_D))\subseteq  \Sigma^{ex}$, then the A-Rule is only applicable on $\texttt{rolE}_{\Sigma}(\OsplitForm)$. 
\end{enumerate}
  By Theorem~\ref{theo:rolE} and \ref{prop_A_rule}, we
    have 
    $$\cl^{ex}(\mathcal{O}) \equiv_{\Sigma^{ex}}\texttt{rolE}_{\Sigma^{ex}}(\cl^{ex}(\mathcal{O})),$$
    $$\texttt{rolE}_{\Sigma^{ex}}(cl^{ex}(\Omc))\equiv_{\Sigma^{ex}}
    \texttt{conE}_{\Sigma^{ex}}(\texttt{rolE}_{\Sigma^{ex}}(cl^{ex}(\Omc))).$$
It follows from these observations that
\begin{align*}
cl^{ex}(\mathcal{O})&\equiv_{\Sigma^{ex}} \texttt{conE}_{\Sigma^{ex}}(\texttt{rolE}_{\Sigma^{ex}}(cl^{ex}(\Omc)))\\
&= \texttt{conE}_{\Sigma}(\texttt{rolE}_{\Sigma}(\OsplitForm))\cup cl(\mathcal{O}_D)\\
&\equiv_{\Sigma} \texttt{conE}_{\Sigma}(\texttt{rolE}_{\Sigma}(\OsplitForm))\cup \mathcal{O}_D.
\end{align*}
Moreover, we have
$$\texttt{conE}_{\Sigma}(\texttt{rolE}_{\Sigma}(\OsplitForm))\cup \mathcal{O}_D\equiv_\Sigma \gm_\Sigma(\mathcal{O})$$
because $\gm_\Sigma(\mathcal{O})$ can be obtained by the following  three operations over $\texttt{conE}_{\Sigma}(\texttt{rolE}_{\Sigma}(\OsplitForm))\cup \mathcal{O}_D$ : 
\begin{enumerate}
    \item Replace all occurrences of the definers $D$ in
$\texttt{conE}_{\Sigma}(\texttt{rolE}_{\Sigma}(\OsplitForm))$ by
$\texttt{copy}_\Sigma(C_D)$ and remove  tautologies.
    
    \item Replace every new introduced concept $\overline{B}$, $\overline{r}$ by $B, r$, respectively.
    Note that
    $\overline{B}, B, \overline{r}, r\not\in \Sigma$ by the definition of $\texttt{copy}_\Sigma$.
     \item Apply exhaustively the translations $C_1\sqsubseteq\neg C_2\sqcup C_3$ $\Rightarrow$
 $C_1\sqcap C_2\sqsubseteq C_3$ and 
 $C_1\sqsubseteq\quant r.\neg C_2\sqcap C_3$ $\Rightarrow$ 
 $C_1\sqcap\overline{\quant}r.C_2\sqsubseteq C_3$, where ${\overline{\exists}}={\forall}$ and ${\overline{\forall}}={\exists}$.
\end{enumerate}
These  operations produce a new ontology that is $\Sigma$-inseparable to
the input ontology.

\noindent
In conclusion, since $\Sigma\subseteq \Sigma^{ex}$, we have $\cl^{ex}(\Omc) \equiv_\Sigma\texttt{conE}_{\Sigma}(\texttt{rolE}_{\Sigma}(\OsplitForm))\cup \mathcal{O}_D\equiv_\Sigma \gm_\Sigma(\mathcal{O})$. This completes the proof.
\end{proof}

\maintheo*
\begin{proof}
By Lemma~\ref{claim_ex_rolE} and~\ref{claim_gm},
we have
$$\mathcal{O}\equiv_\Sigma \cl^{ex}(\Omc) \equiv_\Sigma  \gm_\Sigma(\mathcal{O}),$$
which proves the theorem.
\end{proof}

\ \\

\input{appendix-properties}

\input{appendix-optimizations}
\input{appendix-classical-modules}

%% file: appendix-computing-RI.tex
\subsection{Lemma \ref{theo-preprocessing}}

\newcommand{\NP}{\mathsf{N_P}\xspace}

Fix an ontology $\Omc$, and fix a set $\NP\subseteq
\NC$ of concept names such that
$\NP\cap \sig(\Omc) =\NP\cap \ND = \emptyset$. We
will use the concept names $P\in \NP$ as \emph{placeholders} to
substitute literals of the from $\quant r.D$ in $\cl(\Omc)$. In particular, we
normalize the axioms in $\cl(\Omc)$ further so that every axiom is of one of the
following forms.
\begin{enumerate}
\item $\neg P\sqcup \quant r. D$, where $\quant\in\{\forall,
\exists\}$, and $P\in\NP$;

\item or $L_1\sqcup\ldots\sqcup L_n$, where each $L_i$ is of the forms $A$, or
$\neg B$, where $A\in \NC\setminus\ND,\ B\in \NC\setminus\NP$. 
\end{enumerate}
\hui{I modified the   form of axioms again. We do not allow literals of the form $D, D\in\textsf{N}_D$}
Denote the resulting ontology by $\cl_P(\mathcal{O})$. Similar as with the
normalization introduced in Section~\ref{sec:normalize}, we can ensure that
$\cl_P(\mathcal{O})\equiv_{\sig(\cl(\Omc))}\cl(\Omc)$.

  \begin{table}[t]
    \centering
   \begin{tcolorbox}
  \begin{align*}
    &\mathbf{R^+_A}~ \frac{}{H\sqsubseteq A}: A\in H \qquad\qquad\mathbf{R^-_A}
\frac{H\sqsubseteq N\sqcup A}{H\sqsubseteq N}: \neg A\in H\\[0.7mm]
 &\mathbf{R^n_{\sqcap}}~ \frac{\{H\sqsubseteq N_i\sqcup
A_i\}_{i=1}^n}{H\sqsubseteq \sqcup_{i=1}^n N_i\sqcup N} : (\bigsqcup_{i=1}^n
\neg A_i)\sqcup N\in cl_P(\Omc)\\[0.7mm]
 &\mathbf{R^+_\exists}~ \frac{H\sqsubseteq N\sqcup P}{H\sqsubseteq N\sqcup
\exists r.D} :  \neg P\sqcup \exists r. D\in cl_P(\Omc)\\[0.7mm]
 &\mathbf{R^\bot_\exists}~ \frac{H\sqsubseteq N\sqcup \exists r.K, K\sqsubseteq
\bot}{H\sqsubseteq N}\\[0.7mm]
 &\mathbf{R_\forall}~ \frac{H\sqsubseteq N\sqcup \exists r.K, H\sqsubseteq
N_1\sqcup P}{H\sqsubseteq N\sqcup  N_1\sqcup \exists r.(K\sqcap D)} : \neg P
\sqcup \forall r.D\in cl_P(\Omc)
    \end{align*}
    \end{tcolorbox}
    \caption{Adaptation of the calculus from
\protect\cite{DBLP:conf/ijcai/SimancikKH11}
to the syntax in $\cl_P(\Omc)$.}
    \label{inference rules}
\end{table}

In order to prove Lemma~\ref{theo-preprocessing}, we make use of the inference
calculus introduced in~\cite{DBLP:conf/ijcai/SimancikKH11}, which can be seen in
Table~\ref{inference rules}, adapted to our syntax of normalized axioms. Here, $H$ and $K$ stand for conjunctions of
concept names ( i.e, $A_1\sqcap\ldots \sqcap A_n$), and $N$ for
a disjunction of concept names  ( i.e, $A_1\sqcup\ldots \sqcup A_n$). We write
$$\Omc\vdash_P  H\sqsubseteq \bot$$ if $H\sqsubseteq \bot$ can be
derived from $\Omc$ using the rules in Table~\ref{inference rules}. We have the
following theorem as the adaptation of
\cite[Theorem~1]{DBLP:conf/ijcai/SimancikKH11} to $\cl_P(\mathcal{O})$.
\begin{theorem}
\label{theo:condor}
Let $\Omc$ be an ontology and $H$ be any conjunction of
concept names. Then, $cl_P(\Omc)\models H\sqsubseteq \bot$ iff
$cl_P(\Omc)\vdash_P H\sqsubseteq \bot$.
\end{theorem}

In the following, we write
$$cl_P(\Omc)\vdash \neg D_1\sqcup \ldots \sqcup \neg D_n$$ if we can derive
$\neg D_1\sqcup \ldots \sqcup \neg D_n$ from $cl_P(\Omc)$ using applications
of the A-Rule and the r-Rule.
Before we show Lemma~\ref{theo-preprocessing}, we need an auxiliary lemma.
\begin{lemma}\label{claim1}
Let $N= A_1\sqcup \ldots \sqcup A_m$ be an axiom s.t.
$A_j\in \textsf{N}_C$ for each $1\leq j\leq m$, let $D_1, \ldots,
D_n\in\textsf{N}_D$ be definers, and assume
$cl_P(\Omc)\vdash_P D_1\sqcap\ldots \sqcap D_n\sqsubseteq  N$. Then,
$D_1\sqcap\ldots\sqcap D_n\sqsubseteq N$ is a tautology, or
$cl_P(\Omc)\vdash N'$, where $N'$ is a disjunction over some subset of
$\{\neg D_1,\ldots, \neg D_n, A_1, \ldots, A_n\}$.
%
\end{lemma}
\patrick{I think in the main text, you only say "axiom" and never "clause" - this should be consistent in both parts of the text. In general, I wouldn't object to called the normalized axioms clauses. But then we should introduce this notation in Section 3, and be consistent throughout the text.}
\hui{I changed all clauses to axiom.}
\begin{proof}
Let $k_P(D_1\sqcap\ldots \sqcap D_n\sqsubseteq N)$
be the minimal number of applications of rules from
Table~\ref{inference rules} required to derive $D_1\sqcap\ldots \sqcap
D_n\sqsubseteq N$.
 We prove the lemma by induction on $k_P(D_1\sqcap\ldots \sqcap D_n\sqsubseteq
N)$.

\begin{enumerate}
\item If $k_P(D_1\sqcap\ldots \sqcap D_n\sqsubseteq N)=1$, then we have $N = D_i$ for some $1\leq i\leq n$. In this case, the lemma holds directly.
\item Assume that the lemma holds for all $D_1'$, $\ldots$, $D_{n'}'$, $N'$
s.t.  $k_P(D_1'\sqcap\ldots \sqcap D_{n'}\sqsubseteq
N')<k_0$ for some $k_0\geq 1$.

We show that the lemma also holds for $D_1$, $\ldots$, $D_n$, $N$ s.t.
$k_P(D_1\sqcap\ldots \sqcap
D_n\sqsubseteq N)=k_0$. For simplicity, let $H_D := D_1\sqcap\ldots \sqcap
D_n$, and let $$\rho_1, \ldots, \rho_{k_0}$$ be a sequence of inferences
generated by inference rules in Table \ref{inference rules} that derives
$H_D\sqsubseteq N$. Then there are two different cases, depending on which was
the last inference performed.
\begin{enumerate}
\item The last inference $\rho_{k_0}$ is of the form

$$ \frac{\{H_D \sqsubseteq N_i\sqcup A_i\}_{i=1}^n}{H_D\sqsubseteq
\bigsqcup_{i=1}^n N_i\sqcup N_0},\ \left(\bigsqcup_{i=1}^n \neg A_i\right)\sqcup
N_0 \in cl_P(\Omc).$$
In particular, $N = \bigsqcup_{i=1}^n N_i\sqcup N_0$. We show that we then also
have $\cl_P(\Omc)\vdash D_1\sqcap\ldots \sqcap D_n\sqsubseteq  N$.

If $H_D \sqsubseteq N_{i_0}\sqcup A_{i_0}$ is a tautology for some $1\leq i_0\leq n$, then one of $H_D \sqsubseteq N_{i_0},\ H_D \sqsubseteq A_{i_0}$ must be a tautology. There are two cases:

\textbf{(i)} If  $H_D \sqsubseteq N_{i_0}$  is a tautology, then $H_D \sqsubseteq N$ is also a tautology since $N_{i_0}$ is a sub-concept of $N$.  Therefore, the lemma holds for this case;

\textbf{(ii)} If  $H_D \sqsubseteq A_{i_0}$  is a tautology, then $A_{i_0}\in \{D_1,\cdots, D_n\}$ must be a definer. By the construction of $cl_P(\Omc)$, we have $n=1$ (in the formula of inference $\rho_{k_0}$) and thus $ \neg A_{i_0}\sqcup
N_0 \in cl_P(\Omc)$.  
Then $cl_P\models  \neg A_{i_0}\sqcup
N_0$.  Therefore, the lemma holds also for this case.

We obtain that the lemma holds for the case where $H_D \sqsubseteq N_{i_0}\sqcup A_{i_0}$ is a tautology for some $1\leq i_0\leq n$.

Now assume that $H_D \sqsubseteq N_i\sqcup A_i$ is not a tautology for any $1\leq i\leq n$.

Since $k(H_D\sqsubseteq N_i\sqcup A_i)<k_0$, by applying our
inductive hypothesis on $H_D\sqsubseteq N_i\sqcup A_i$, $i\in[1, n]$, we have
\begin{align*}
  &cl_P(\Omc)\vdash K_D^{i}\sqcup  N_i^{sub}\\
  \text{ or }\ \ &cl_P(\Omc)\vdash K_D^{i}\sqcup  N_i^{sub}\sqcup A_i
\end{align*}
for some $K_D^{i} = \bigsqcup_{j\in I_i}\neg D_j$,  $I_i\subseteq \{1,\ldots ,
n\}$ and $N^{sub}_i$ being a disjunction over some concept names
occurring in~$N_i$.  We distinguish those two cases.

\textbf{(i)} If $cl_P(\Omc)\vdash K_D^{i}\sqcup  N_i^{sub}$ for some $1\leq
i\leq n$, then $cl_P(\Omc)\vdash K_D^{i}\sqcup  N_i^{sub} $ is as desired.

\textbf{(ii)} Otherwise, $cl_P(\Omc)\vdash K_D^{i}\sqcup  N_i^{sub}\sqcup A_i$
for all $1\leq i\leq n$. By applying the A-Rule for all $A_i, 1\leq i\leq n$ on
\begin{align*}
K_D^{i}&\sqcup  N_i^{sub}\sqcup A_i,\ \ \ (1\leq i\leq n)\\
\text{ and }&(\bigsqcup_{i=1}^n \neg A_i)\sqcup N_0\in cl_P(\Omc),
\end{align*}
we obtain the desired conclusion
$$cl_P(\Omc)\vdash K_D \sqcup  N^{sub},  $$ 
where 
$$K_D= \bigsqcup_{1\leq i\leq n}K_D^{i},\text{ and }$$  
$$N^{sub}= \bigsqcup_{1\leq i\leq n}N^{sub}_{i}\sqcup N_0.$$

Therefore, the lemma holds for this case.
\item The last inference $\rho_{k_0}$ is generated by Rule~$\mathbf{R^\bot_\exists}$ and is of the form
$$\frac{H_D\sqsubseteq N\sqcup \exists r.K,\ K\sqsubseteq \bot}{H_D\sqsubseteq N}.$$
Note that $H_D\sqsubseteq N\sqcup \exists r.K$ must be obtained by applying 
\begin{itemize}
\item first an $\mathbf{R^+_\exists}$ inference of the form
$$ \frac{H_D\sqsubseteq N_0\sqcup P_0}{H_D\sqsubseteq N_0\sqcup \exists r.D_0'} :  \neg P_0\sqcup \exists r. D_0'\in cl_P(\Omc);$$

\item followed by $m$ $\mathbf{R_\forall}$ inferences of the form
$$\frac{H_D\sqsubseteq (\bigsqcup\limits_{i=0}^{j-1}N_i)\sqcup \exists r.K_{j-1},\ H_D\sqsubseteq N_j\sqcup P_j}{H_D\sqsubseteq (\bigsqcup\limits_{i=0}^{j}N_i)\sqcup \exists r.(K_{j-1}\sqcap D_j')},$$
where $\neg P_j \sqcup \forall r.D_j'\in cl_P(\Omc)$ and $K_j =
D_0'\sqcap\ldots \sqcap D_{j}'$ for $1\leq j\leq m$. Moreover, we have
$$N = (\bigsqcup\limits_{i=0}^{m}N_i),\ \ K = K_{m}.$$
\end{itemize}

By applying the inductive hypothesis on $K\sqsubseteq \bot$, we obtain
$$cl_P(\Omc)\vdash \bigsqcup_{i\in I^*} \neg D_i', \text{ for some } I^*\subseteq \{0,\ldots m\},$$
and by applying the inductive hypothesis on $H_D\sqsubseteq N_j\sqcup P_j$,
$j\in[0, m]$, we obtain
\begin{align*}
  &cl_P(\Omc)\vdash H_D^{j}\sqcup  N_j^{sub},\\
  \text{ or }\ \ &cl_P(\Omc)\vdash H_D^{j}\sqcup  N_j^{sub}\sqcup P_j,
\end{align*}
for some $H_D^{j} = \bigsqcup_{j\in I_i}\neg D_j$ with  $I_i\subseteq \{1,\ldots
, n\}$ and $N^{sub}_j$ a disjunction of concepts from $N_j$,

We again distinguish both cases.

\textbf{(i)} If $cl_P(\Omc)\vdash H_D^{j}\sqcup  N_j^{sub}$ for some $0\leq
j\leq m$, then $cl_P(\Omc)\vdash H_D^{j}\sqcup  N_j^{sub}$ directly holds.

\textbf{(ii)} Otherwise, $cl_P(\Omc)\vdash H_D^{j}\sqcup  N_j^{sub}\sqcup P_j$
for all $0\leq j\leq m$. By
  applying the A-Rules for all $P_j,\ 1\leq j\leq m$ on

\begin{align*}
H_D^0\sqcup N_0^{sub}\sqcup P_0, \ \ &\neg P_0\sqcup \exists r. D_0'\\
H_D^j\sqcup N_j^{sub}\sqcup P_j,\ \ &\neg P_j \sqcup \forall r.D_j',\\
&(1\leq j\leq m);
\end{align*}
and applying the r-Rule on
\begin{align*}
&\ \ \ \ \ \  \bigsqcup_{i\in I^*} \neg D_i',\\
&H_D^0\sqcup N_0^{sub}\sqcup \exists r. D_0',\\
&H_D^j\sqcup N_j^{sub}\sqcup \forall r.D_j',\\
&\ \ \ \ \ \ (j\in I^*\cap\{1,\ldots, m\}),
\end{align*}

we obtain
$$cl_P(\Omc)\vdash H_D \sqcup N^{sub},$$
where 
$$H_D = \bigsqcup_{j\in I^*\cup\{0\} }H_D^{j},\text{ and}$$
$$N^{sub}= \bigsqcup_{j\in I^*\cup\{0\}}N^{sub}_{j}.$$
\end{enumerate}
We obtain that the lemma also holds in this case.\qedhere
\end{enumerate}
\end{proof}

Using Theorem~\ref{theo:condor} and Lemma~\ref{claim1}, we can now prove Lemma~\ref{theo-preprocessing}.

\lemmaPreprocessing*
\begin{proof}
For any definers $D_1,\ldots, D_n\in\ND$, since $\cl_P(\mathcal{O})\equiv_{\sig(\cl(\Omc))}\cl(\Omc)$, we have $\cl(\Omc)\models D_1\sqcap\ldots \sqcap D_n\sqsubseteq \bot$ iff $\cl_P(\Omc)\models D_1\sqcap\ldots \sqcap D_n\sqsubseteq \bot$.

Note that we can exchange the order of application of a A-Rule on concept name $P\in \NP$ and other rules without influencing the final result. For instance:
the following two rules that produce $C_1\sqcup \ldots\sqcup C_n\sqcup \quant r. D$.

$$\text{(r-Rule):} \cfrac{C_1\sqcup P\sqcup \exists r. D_1,\  \bigcup_{j=2}^{n}\{ C_j\sqcup \forall r. D_j\},\ K_D}{C_1\sqcup \ldots\sqcup C_n\sqcup P},\ \ \text{(A-Rule on $P$):} \cfrac{C_1\sqcup \ldots\sqcup C_n\sqcup P, \ \neg P\sqcup \quant r. D}{C_1\sqcup \ldots\sqcup C_n\sqcup \quant r. D} $$

 We also obtain $C_1\sqcup \ldots\sqcup C_n\sqcup \quant r. D$  by following two rules, where  a A-Rule on $P\in \NP$ is applied first.
$$\text{(A-Rule on $P$):} \cfrac{C_1\sqcup P\sqcup \exists r. D_1, \ \neg P\sqcup \quant r. D}{C_1\sqcup \quant r. D \sqcup \exists r. D_1} ,\ \ \text{(r-Rule):} \cfrac{C_1\sqcup \quant r. D \sqcup \exists r. D_1,\   \bigcup_{j=2}^{n}\{ C_j\sqcup \forall r. D_j\},\ K_D}{C_1\sqcup \ldots\sqcup C_n\sqcup \quant r. D}$$

Therefore, when applying A-Rule and r-Rule on $\cl_P(\Omc)$, we could assume that A-Rules on concept names $P\in \NP$ are applied first. Since applying  A-Rules on concept names $P\in \NP$ on $\cl_P(\Omc)$ produces exactly the axioms in $\cl(\Omc)\setminus \cl_P(\Omc)$, we have $\cl(\Omc)\vdash \neg D_1\sqcup \ldots \sqcup \neg D_n$ iff $\cl_P(\Omc)\vdash \neg D_1\sqcup \ldots \sqcup \neg D_n$ for any definers $D_i\in\ND$. 

It is thus enough to show that  for any definers $D_i\in\ND$, 
$\cl_P(\Omc)\vdash_P D_1\sqcap\ldots \sqcap D_n\sqsubseteq \bot$ iff $\cl_P(\Omc)\vdash \neg D_{i_1}\sqcup \ldots \sqcup \neg D_{i_k}$ for some subset $ \{i_1,\ldots,i_k\}\subseteq \{1,\ldots , n\}$. 

We first prove the ``$\Leftarrow$'' direction. If $cl(\Omc)\vdash \neg D_{i_1}\sqcup \ldots \sqcup \neg D_{i_k}$, then we have $cl(\Omc)\models D_1\sqcap\ldots \sqcap D_n\sqsubseteq \bot$. Consequently, by Theorem~\ref{theo:condor}, $cl_P(\Omc)\models D_1\sqcap\ldots \sqcap D_n\sqsubseteq \bot$, and thus $cl_P(\Omc)\vdash_P D_1\sqcap \ldots \sqcap D_n\sqsubseteq\bot$.

The ``$\Rightarrow$'' direction  is a direct result of Lemma~\ref{claim1}.
\end{proof}

Moreover, we have the following lemma that will be used in the proof of Theorem \ref{theo:Sigma_normalized} in the next section.
\begin{restatable}{lemma}{corMr}
\label{corMr}
Let $t\in\NR\setminus\Sigma$ be a role name and $D_1$, $\ldots$, $D_n\in\ND$ be
definers s.t.
$\OsplitForm$ contains a literal of the form $\quant t.D_1$, and for $2\leq
i\leq n$, $\forall t.D_i$ occurs in $\OsplitForm$. Then,
$\cl(\mathcal{O})\models D_1\sqcap\ldots\sqcap
D_n\sqsubseteq\perp$ iff $\OsplitForm\models D_1\sqcap\ldots\sqcap
D_n\sqsubseteq\perp$.
\hideThisPart{$D_2$, $\ldots$, $D_n$ occur under value restrictions,
If definers $D_1,\ D_2,\ \ldots,\ D_n\in \sig(\OsplitForm)$ satisfy
$\bigcap_{i=1}^n \Rol(D_i,\OsplitForm)\neq \emptyset$, then we have
$\cl(\mathcal{O})\models D_1\sqcap\ldots\sqcap
D_n\sqsubseteq\perp$ iff $\OsplitForm\models D_1\sqcap\ldots\sqcap
D_n\sqsubseteq\perp$.
}
\end{restatable}
\begin{proof}
If $\OsplitForm\models D_1\sqcap\ldots\sqcap D_n\sqsubseteq\bot$, then also
$\cl(\mathcal{O})\models D_1\sqcap\ldots\sqcap D_n\sqsubseteq\bot$ because
$\OsplitForm$ consists only of axioms from $\cl(\Omc)$ or axioms that have been
derived from $\cl(\Omc)$. The other direction follows directly from the
definition of $\deSigmaPart$ in Definition~\ref{def:SigmaForm}.
\end{proof}

\hui{I just notice that it is enough to prove the lemma for the case
$\bigcap_{i=1}^n \Rol(D_i,\OsplitForm) = \{t\}\not\subseteq \Sigma$. Because we
change the proof of Theorem \ref{theo:Sigma_normalized} using Theorem
\ref{theo:r-forget}. This lemma is applied only in the Item 2 in proof of
Theorem \ref{theo:Sigma_normalized}.}
\patrick{With this additional side condition, and changing something else
that was not correct in the formulation of the lemma, the proof simplifies
completely.}
\hideThisPart{
\begin{proof}
Let $t$, $D_1$, $\ldots$, $D_n$ be as in the lemma.
We only have to prove: if $\cl(\mathcal{O})\models D_1\sqcap\ldots\sqcap
D_n\sqsubseteq\bot$ then $\OsplitForm\models D_1\sqcap\ldots\sqcap
D_n\sqsubseteq\bot$. The other direction follows directly from
Definition~\ref{def:SigmaForm}.
\hideThisPart{
Because in our normalization, we assume that different occurrences of concepts
are always replaced by distinct definers, we have that for each definer $D\in
\sig(\clSigmaPart)$, $\Rol(D, \clSigmaPart)$ contains exactly one role name.
\patrick{Check throughout the text: do we say ``role'' or ``role name''? If
different, make consistent.}
\hui{Corrected.}
Consequently, since $\bigcap_{i=1}^n \Rol(D_i,\OsplitForm)\neq
\emptyset$, there must exists a single role name $t\in\NR$ such that
\begin{align*}
    \bigcap_{i=1}^n \Rol(D_i,\OsplitForm) &= \Rol(D_1,\OsplitForm)\\
    &= \Rol(D_2,\OsplitForm)\\
    &\phantom{=\ } \vdots\\
    &= \Rol(D_n,\OsplitForm) \\
    &= \{t\}.
\end{align*}
Now if $t\not\in\Sigma$, then by the definition of $\deSigmaPart$ in
Definition~\ref{def:SigmaForm}, $\neg D_{i_1}\sqcup \ldots\sqcup \neg D_{i_n}\in
\deSigmaPart\subseteq \clSigmaPart$, where $\{i_1, \ldots i_m\}\subseteq \{1,
\ldots, n\}$.
\patrick{That is not fully correct: the definers may simply occur in different
ways. }
Therefore, the interesting case is where $t\in\Sigma$. Note that
 in this case, for no $i\in\{1,\ldots,n\}$, the definer $D_i$
occurs in $\deSigmaPart$.
}
If $t\not\in\Sigma$, then by the definition of $\deSigmaPart$ in
Definition~\ref{def:SigmaForm}, $\neg D_{i_1}\sqcup \ldots\sqcup \neg D_{i_n}\in
\deSigmaPart\subseteq \clSigmaPart$, where $\{i_1, \ldots i_m\}\subseteq \{1,
\ldots, n\}$. Therefore, the interesting case is where $t\in\Sigma$. Note that
 in this case, for no $i\in\{1,\ldots,n\}$, the definer $D_i$
occurs in $\deSigmaPart$.

Assume that $t\in\Sigma$. We then have $\Rol(D_i,\clSigmaPart)\subseteq \Sigma$
for all $1\leq i\leq n$.
Assume further $\cl(\mathcal{O})\models D_1\sqcap\ldots\sqcap D_n\sqsubseteq\bot$. By
Lemma~\ref{theo-preprocessing}, we can derive, using the A-Rule and the r-Rule, an axiom $\neg D_{i_1}\sqcup \ldots\sqcup \neg D_{i_k} \in
\mathcal{S}$, where $\{i_1,\ldots, i_k\}\subseteq \{1, \ldots, n\}$.
Without loss of generality, we assume that $\{i_1,\ldots, i_k\}= \{1, \ldots, n\}$. 

%
%
Consider a derivation tree $T$ showing the derivation of $\neg
D_1\sqcup\ldots\sqcup \neg D_n$ from
$\cl(\mathcal{O})$ using the A-Rule and the r-Rule as shown in
Figure~\ref{derivation-tree}.
Here, nodes correspond to axioms, and edges to applications of the A-rule or
the r-rule, so that the root is $\neg D_1\sqcup\ldots\sqcup D_n$ and the leafs
are axioms from $\cl(\Omc)$.
We create a new derivation tree $T_1$ by removing
all children of nodes of the form $\neg D_1'\sqcup \ldots\sqcap \neg D_m' \in \OsplitForm$ from $T$, so that those nodes become leafs. 
Clearly, the root of $T_1$ is still $\neg D_1\sqcup\ldots\sqcup \neg D_n$. We
argue that $T_1$ shows a derivation of $\neg D_1\sqcup\ldots\sqcup
\neg D_n$ from $\OsplitForm$.
For this, we show that every leaf node of $T_1$ occurs in
$\OsplitForm$.

\begin{figure}
\centering
\scalebox{0.8}{\begin{tikzpicture}
[
    block/.style ={rectangle, draw=black, thick, fill=white!20, text width=8em,align=center, rounded corners, minimum height=1.5em},
    level 1/.style = {black, sibling distance = 4.5cm},
    level 2/.style = {black, sibling distance = 1cm},
    level 3/.style = {black, sibling distance = 3.5cm},
        level 4/.style = {black, sibling distance = 1.5cm},
    edge from parent fork down
]
\node (root)[block]{ $\neg D_{i_1}\sqcup \ldots\sqcup \neg D_{i_k}$}
    child {node(c1)[block] {$\textit{axiom}_1$}
    edge from parent [<-]
        child {node {\large\bf$\ldots$}
            edge from parent [<-]
            child {node [block] {$\neg D'_1\sqcup\ldots\sqcup \neg D'_m$}
                edge from parent [<-]
                child {node {\large\bf$\ldots$}}
                child {node {\large\bf$\ldots$}}
                }
            child {node  {\scriptsize \bf $\vdots$}}
            }
        }
    child {node[block] {$\textit{axiom}_n$}
        edge from parent [<-]
        child {node {\scriptsize \bf $\vdots$ }
         }
        child {node {\scriptsize \bf $\vdots$ } }
    };
    \path (0,-0.95) node (f1) {\textcolor{teal}{r-Rule or A-Rule} };
    \path (0,-1.5) node (f2) {\large\bf$\ldots$};
    \path (-1.7,-4.5) node (f3) {\textcolor{teal}{$\in \deSigmaPart$} };

    \filldraw[thick,dashed, blue, fill=white,fill opacity=0] (-6,-5) -- (-6,0.5) -- (4,0.5) -- (4,-6.5) -- (-1.5,-6.5) -- (-1.5,-5) -- (-6,-5) -- (-6,-6.5)-- (-1.5,-6.5);
    \path (-5.5,0) node (f4) {\textcolor{red}{$T_1$} };
   \path (-5.5,-5.5) node (f4) {\textcolor{red}{Cut}};
\end{tikzpicture}}
    \caption{Derivation tree $T$ ($\{i_1,\ldots, i_k\}\subseteq \{1, \ldots, n\}$)}\patrick{Please update the picture - I believe the notation is not up-to-date anymore.}
    \patrick{Check: The $n$ on the axioms shouldn't be the same $n$ as for the definers, right? The root could have been inferred using less axioms. }
    \hui{Corrected}
    \label{derivation-tree}
\end{figure}

For a proof by contradiction, assume that $T_1$ has a leaf $\alpha'$ s.t.
$$\alpha'\in \cl(\mathcal{O})\setminus \OsplitForm.$$
Then, also $\alpha'\not\in\clSigmaPart$.
By the definition of $\clSigmaPart$, this means that $\alpha'$ must contain
some literal $\neg D_1'$ such that, for some definer $D^\dagger\in\sig(\cl(\Omc))$,
$D_1'\preceq_d D^\dagger$ and $\Rol(D^\dagger,\cl(\Omc))\not\subseteq \Sigma$.
In particular,
it means that for any definer $D'$ s.t. $D_1'\preceq_d D'\preceq_d D^\dagger$,
$\clSigmaPart$
also cannot contain an axiom with the literal $\neg D'$.
By the definition of $\preceq_d$, this means that, unless $D_1'=D^\dagger$,
there also cannot be a clause in $\clSigmaPart$ that
contains $D_1'$, which means that $D_1'\not\in\sig(\clSigmaPart)$.
\patrick{Unless $D_1'= D^\dagger$. Also: explain why!}
Note also that $D_1'$ cannot be a disjunct of the root in $T_1$:
 we already ruled out the case that the root contains disjuncts $D_i$ s.t. $D_i\in\sig(\deSigmaPart)$, but by the assumptions of the lemma, every $D_i$ must occur 
in $\OsplitForm=\clSigmaPart\cup\deSigmaPart$. It follows that $D_1'$ must contribute within the derivation in $T_1$ to an axiom that does not contain $D_1'$.

By our assumptions on $\cl(\Omc)$, there cannot be an axiom containing the positive literal $D_1'$.
Therefore, the only way to eliminate the literal $D_1'$ within the derivation
tree is using the r-Rule, which means that $\alpha'$ must contribute to the
derivation of an axiom $\neg D_1'\sqcup\ldots\sqcup\neg D_k'$ occurring in
$T_1$, s.t.
$\Rol(D_1',\cl(\Omc))=\Rol(D_2',\cl(\Omc))=\ldots=\Rol(D_k',\cl(\Omc))=\{r\}$.
We note that $r\in\Sigma$, since otherwise $\neg D_1'\sqcup\ldots\sqcup\neg
D_k'\in\deSigmaPart$, and $\alpha'$ would have been removed from $T$ together
with all other descendants of $\neg D_1'\sqcup\ldots\sqcup\neg D_k'$ when
creating $T_1$. One consequence of this is that, since
$\Rol(D_1',\cl(\Omc))\subseteq\Sigma$ and
$\Rol(D^\dagger,\cl(\Omc))\not\subseteq\Sigma$, we must have $D^\dagger\neq
D_1'$.
\patrick{This comes too late!}

$\neg D_1'\sqcup\ldots\sqcup \neg D_k'$ is now used in $T_1$ to infer a new axiom $\alpha''$ with the r-Rule. This new axiom must contain a literal of the form $\neg D'$, since we know that there is the definer $D^\dagger\neq D_1'$ s.t. $D_1'\preceq_d D^\dagger$, and we also know that $D'$, for the same reasons as for $D_1'$, cannot occur in $\clSigmaPart$ or in the root of $T_1$. 
But then, we can apply for $D'$ the same argument as for $D_1'$, and obtain that we can never reach the final conclusion $\neg D_1\sqcup\ldots\neg D_n$ of the derivation $T_1$, which means that our axiom $\alpha'$ cannot have been part of the derivation of $T_1$ to begin with. 
\patrick{I got it now, but I think my reformulated version of the proof is still
not very easy to follow. We should spell out what implications our preorder
$\preceq_d$ gives us on $\OsplitForm$. I think this may help. I will take care
of it at a later point. }

We obtain that every leaf node of $T_1$ occurs in $\OsplitForm$, and thus that $T_1$ shows how to derive $\neg D_1\sqcup\ldots\sqcup \neg D_n$
from $\OsplitForm$ using the A-Rule and the r-Rule.
As a consequence, $\OsplitForm\models D_1\sqcap\ldots\sqcap
D_n\sqsubseteq\bot$. 
%
%
\end{proof}
}

\subsection{Theorem \ref{theo:Sigma_normalized}}
\begin{figure}
    \centering
   \begin{tcolorbox}
\begin{align*}
 \textbf{r-Res}: &\frac{C_1\sqcup \exists r. D_1,\ C_2\sqcup \forall r. D_2,\ \ldots, C_n\sqcup \forall r. D_n}{C_1\sqcup \ldots\sqcup C_n},\\
&\text{ where }\mathcal{M}\models D_1\sqcap \ldots\sqcap D_n\sqsubseteq \bot, n\geq 1.
\end{align*}
    \end{tcolorbox}
    \caption{Rule for eliminating role name $r$.}
    \patrick{Why do you keep calling your rules ``resolution rule''? The
resolution rule is a specific rule for propositional or first-order that
resolves a negative literal in one axiom with a postive literal in the other
axiom. If your rule is not such a rule, I would just call it ``rule'' or
``inference rule''.}
\hui{Got it. I will rename them as rule}
    \label{fig:LETHE-rule}
\end{figure}

In order to prove Theorem~\ref{theo:Sigma_normalized}, we make use of a result for role 
forgetting from~\cite{DBLP:conf/aaai/ZhaoASFSJK19}, which describes the method for 
uniform interpolation used in \Fame. 
In~\cite{DBLP:conf/aaai/ZhaoASFSJK19}, the rule shown in Figure~\ref{fig:LETHE-rule} is used
for forgetting roles, which assumes the ontology to be in normal form, and uses a fixed 
ontology $\Mmc$ as a side condition.

Fix a signature $\Sigma\subseteq \sig(\Omc)$.
We denote by $\Res_\Sigma (\cl(\mathcal{O}), \mathcal{M} )$ the ontology
obtained by applying the following two operations on $\cl(\mathcal{O})$:
\begin{enumerate}
\item apply \textbf{r-Res} exhaustively for all role names
$r\in\sig(\Omc)\setminus\Sigma$, and
\item remove all axioms that contain a role name $r\in\sig(\Omc)\setminus\Sigma$.
\end{enumerate}

By~\cite[Lemma 3]{DBLP:conf/aaai/ZhaoASFSJK19}, we have the following result.
\begin{lemma}\label{theo:r-forget}
For any ontology $\mathcal{O}$ and signature $\Sigma\subseteq
\sig(\Omc)$, we have $\Res_\Sigma(\cl(\Omc), \cl(\Omc))\equiv_{\Sigma\cup
\sigC(\mathcal{O})} \cl(\Omc)$.
\end{lemma}

\patrick{Since the only relevant case is where $\Mmc=\cl(\Omc)$, we can probably simplify notations here.}
\hui{The proof of theorem depends on the case $\Mmc\neq \cl(\Omc)$.}
\patrick{I see. }
Let 
$\Res_\Sigma(\cl(\mathcal{O}), \mathcal{M} )\big|_{\Sigma}$ be the sub ontology of 
$\Res_\Sigma (\cl(\mathcal{O}), \mathcal{M} )$
 that contains only those axioms $\alpha\in\Res_\Sigma (\cl(\mathcal{O}), \mathcal{M})$ that
satisfy:
\begin{itemize}
\item if $L:=\neg D$ is a literal of $\alpha$, then for any definer $D'\in \sig(\cl(\Omc))$ such that $D\preceq_d D'$, we have
 $\Rol\big(D', \cl(\Omc)\big)\subseteq\Sigma$.
\end{itemize}
To prove Theorem~\ref{theo:Sigma_normalized}, we also need the following lemma.

\begin{lemma}\label{lemma:restrictSigma}
    For any ontology $\mathcal{O}$ and signature
    $\Sigma\subseteq \sig(\Omc)$, we have $$
        \Res_\Sigma(\cl(\Omc),\cl(\Omc))
        \ \equiv_{\Sigma\cup \sigC(\mathcal{O})}\
        \Res_\Sigma(
        \cl(\mathcal{O}), \cl(\Omc) )\big|_{\Sigma}.
    $$
\end{lemma}

\begin{proof}
Recall that we assume that each definer occurs at most once positively 
in $\cl(\Omc)$. In particular, for each definer $D\in\sig(cl(\Omc))$, there is at most one occurrence of a literal of the form $\quant r.D$ in $\cl(\Omc)$.

For any definer $D\in \sig(\cl(\Omc))$, if there exists  $D'\in \sig(\cl(\Omc))$ such that $D\preceq_d D'$ and $\Rol\big(D', \cl(\Omc)\big) = \{r_0\}\not\subseteq\Sigma$, then we can find a sequence of axioms in $\cl(\Omc)$ such as the following.
\begin{align*}
\alpha_0:\ \  &C_0\sqcup \quant_0 r_0. D_0,\\
\alpha_1:\ \ & \neg D_0\sqcup  C_1\sqcup \quant_1 r_1.D_1, \\
&\ldots,\\
\alpha_n:\ \  &\neg D_n\sqcup C_{n+1}\sqcup \quant_n r_n.D_{n+1}.\\
\alpha_{n+1}:\ \  &\neg D_{n+1}\sqcup C_{n+2}.
\end{align*}
where $D_0 = D', D_{n+1} = D$. Then, for every $1\leq i\leq n$, $\quant_i r_i. D_i$ is the unique literal containing $D_i$ positively, and $\quant_i r_i. D_i$ appears only in $\alpha_i$. 

Since $r_0\not\in\Sigma$, there is no axiom in $\Res_\Sigma(\cl(\Omc), \cl(\Omc))$ that contains a literal of the form $\quant_0 r_0.D_0$. Therefore, $D_0$ does not appear positively in $\Res_\Sigma(\cl(\Omc), \cl(\Omc))$. It is well known that, if a concept 
$D_0$ occurs only negatively in an ontology, we can preserve all entailments of
axioms not using $D_0$ if we replace $D_0$ by $\bot$, which with our normal form
means, we can delete all axioms with the literal $\neg D_0$ without losing any
consequences in $\Sigma$ (see also \cite[Theorem 1]{DBLP:conf/ijcai/ZhaoS17}).
%
Specifically, if we set $\Sigma_{Res} = \sig(\Res_\Sigma(\cl(\Omc), \cl(\Omc)))$, and let $\Res_\Sigma^0(\Omc)$ be the ontology obtained from $\Res_\Sigma(\cl(\Omc), \cl(\Omc))$ by removing all axioms that contain the literal $\neg D_0$, then we have
$$\Res_\Sigma(\cl(\Omc), \cl(\Omc)) \equiv_{\Sigma_{Res}\setminus\{D_0\}} \Res_\Sigma^0(cl(\Omc)).$$

Note that if a literal $L_1 = \quant r .D$ always occurs in $\cl(\Omc)$ together with another
literal $L_2$ which is of the form $A$ or $\neg A$  (i.e., every axiom $\alpha\in
\cl(\Omc)$ either contains both $L_1$ and $L_2$ or none of them), then $L_1$ also
always appears together with  $L_2$ in $\Res_\Sigma(\cl(\Omc), \cl(\Omc))$.
Because the rule r-Res preserves all literals of the form $A$ or $\neg A$,
and since $\quant_1 r_1.D_1$ always appears with $\neg D_0$ in  $\cl(\Omc)$, $D_1$
cannot appear positively in  $\Res_\Sigma^0(\cl(\Omc))$, because we removed all
occurrences of $D_1$ along with the occurrences of $\neg D_0$. Then, if
$\Res_\Sigma^1(\Omc)$ is the ontology obtained  by removing all axioms
containing
the literal $\neg D_1$ from  $\Res_\Sigma^0(\Omc)$, we have
$$ \Res_\Sigma^0(\cl(\Omc)) \equiv_{\Sigma_{Res}\setminus\{D_1\}} \Res_\Sigma^1(\cl(\Omc)).$$
Consequently, we have 
$$ \Res_\Sigma(\cl(\Omc), cl(\Omc))\equiv_{\Sigma_{Res}\setminus\{D_0, D_1\}} \Res_\Sigma^1(\cl(\Omc)).$$

Repeat this process for all $D_2, \ldots, D_{n+1}$, and we have for $0\leq i\leq n$ that
$$\Res_\Sigma^i(\cl(\Omc)) \equiv_{\Sigma_{Res}\setminus\{D_{i+1}\}} \Res_\Sigma^{i+1}(\cl(\Omc))\text{ and}$$
$$ \Res_\Sigma(\cl(\Omc), cl(\Omc)) \equiv_{\Sigma_{Res}\setminus\{D_0, \ldots, D_{i+1}\}} \Res_\Sigma^{i+1}(\cl(\Omc)), $$
where $\Res_\Sigma^{i+1}(\Omc)$ is the ontology obtained from  $\Res_\Sigma^{i}(\Omc)$ by removing all axioms containing the literal $\neg D_{i+1}$. We conclude that removing all axioms that contain the literal $\neg D_{n+1} = \neg D$ preserves all logical consequences over $\Res_\Sigma(\cl(\Omc), \cl(\Omc))$ in the signature that excludes $\{D_0, \ldots, D_{n+1}\}$.

Finally, we have $$\Res_\Sigma(\cl(\Omc), \cl(\Omc))\equiv_{\Sigma\cup \sigC(\mathcal{O})} \Res_\Sigma( \cl(\mathcal{O}), \cl(\Omc) )\big|_{\Sigma}$$
by repeating the process above for all $D$ such that there exists  $D'\in \sig(\cl(\Omc))$ such that $D\preceq_d D'$ and $\Rol\big(D', cl(\Omc)\big) = \{r_0\}\not\subseteq\Sigma$.
\end{proof}

\bigskip

We now have everything to prove Theorem~\ref{theo:Sigma_normalized}.
\theoSigmanormalized*
\begin{proof}
$\OsplitForm$ is \sigmaForm by definition. We show that  
$\mathcal{O}\equiv_{\Sigma\cup \sigC(\mathcal{O})} \OsplitForm.$

By Theorem~\ref{theo:r-forget} and Lemma~\ref{lemma:restrictSigma}, it is sufficient to show that 
\begin{equation}\label{eq:same-UI}
\begin{aligned}
 & \Res_\Sigma(cl(\mathcal{O}), cl(\mathcal{O}))\big|_{\Sigma} 
  =\ \Res_\Sigma(\OsplitForm, \OsplitForm)\big|_{\Sigma} .
 \end{aligned}
\end{equation}
This follows from the following observations.
 \begin{enumerate}
     \item If an axiom $\alpha$ contains the literal $\neg D$, any axiom obtained
from $\alpha$ using \textbf{r-Res} also
contains $\neg D$.
Therefore, we have after Definition~\ref{def:SigmaForm} defining $\clSigmaPart$:
$$\Res_\Sigma(cl(\mathcal{O}), cl(\mathcal{O}))\big|_{\Sigma}=
\Res_\Sigma(\clSigmaPart, cl(\mathcal{O})).$$

     \item      
     Assume that $C_1\sqcup \exists r. D_1,$ $C_2\sqcup \forall r. D_2$,
$\ldots$, $C_n\sqcup \forall r. D_n \in \Res_\Sigma(\clSigmaPart,
\cl(\mathcal{O}))$ and $r\not\in\Sigma$.
     Then, by Lemma \ref{corMr},  we have $\cl(\mathcal{O})\models
D_1\sqcap\ldots\sqcap D_n\sqsubseteq\bot$ iff $\OsplitForm \models
D_1\sqcap\ldots\sqcap D_n\sqsubseteq\bot$.
     Therefore, for a given premise set $P = \{C_1\sqcup \exists r. D_1,\ C_2\sqcup \forall r. D_2,\ \ldots, C_n\sqcup \forall r. D_n\}\subseteq \Res_\Sigma(\clSigmaPart, \cl(\mathcal{O}))$, an \textbf{r-Res} inference $\rho$ is applicable on $P$ with $\mathcal{M} = \cl(\mathcal{O})$ iff $\rho$ is applicable on $P$ with $\mathcal{M} = \OsplitForm$.
     Consequently, we have\
     $$ \Res_\Sigma(\clSigmaPart, cl(\mathcal{O})) = \Res_\Sigma(\clSigmaPart, \OsplitForm).$$

     \item Since $\OsplitForm = \clSigmaPart\cup \mathcal{D}_{\Sigma}(cl(\mathcal{\Omc}))$ and every axiom in $\mathcal{D}_{\Sigma}(cl(\mathcal{\Omc}))$  has a literal $\neg D$ with $Rol(D, cl(\Omc))\not\subseteq \Sigma$, we have 
     \begin{align*}
      &\Res_\Sigma(\clSigmaPart, \OsplitForm)\\
      =\ &\Res_\Sigma(\OsplitForm, \OsplitForm)\big|_\Sigma.\qedhere
      \end{align*}
 \end{enumerate}
%
\end{proof}

%% file: appendix-role-elimination.tex
\subsection{Theorem~\ref{theo:rolE}}
\hideThisPart{
To prove Theorem~\ref{theo:rolE}, we will show inseparability on the level of interpretations. Given two interpretations $\Imc$, $\Jmc$ and a signature~$\Sigma$, we write $\Imc=_\Sigma\Jmc$ if for all $X\in\Sigma$, $X^\Imc=X^\Jmc$. 

The following is a well-known result regarding deductive inseparability, which we just prove here for matter of self-containment. 
\begin{lemma}\label{lem:inseparability}
    Given two ontologies $\Omc_1$, and $\Omc_2$ and a signature $\Sigma$, we have $\Omc_1\equiv_\Sigma\Omc_2$ if for every model $\Imc_1$ of $\Omc_1$, there is a model $\Imc_2$ of $\Omc_2$ s.t.  $\Imc_1=_\Sigma\Imc_2$, and for every model $\Imc_2$ of $\Omc_2$, there is a model $\Imc_1$ of $\Omc_1$ s.t. $\Imc_2=_\Sigma\Imc_1$.
\end{lemma}
\begin{proof}
Assume the lemma holds for $\Omc_1$, $\Omc_2$ and $\Sigma$. By symmetry, it suffices to show that for every $\Sigma$-axiom $\alpha$, if $\Omc_1\not\models\alpha$, then also $\Omc_2\not\models\alpha$. Assume $\Omc_1\not\models\alpha$. There is then a model $\Imc_1$ of $\Omc_1$ s.t. $\Imc_1\not\models\alpha$, and by the assumption of the lemma, a model $\Imc_2$ of $\Omc_2$ s.t. $\Imc_1=_\Sigma\Imc_2$. Since $X^{\Imc_1}=X^{\Imc_2}$ for every $X\in\sig(\alpha)$, it follows by structural induction on $\alpha$ that also $\Imc_2\not\models\alpha$, and thus $\Omc_2\not\models\alpha$.
\end{proof}
}
\newcommand{\Out}{\textit{Out}}

\theorolE*
\begin{proof} 
Assume $\Omc$ is role isolated for $\Sigma$. Recall that by Lemma~\ref{theo:r-forget},
for any ontology $\mathcal{O}$ and signature $\Sigma\subseteq
\sig(\Omc)$, we have $\Res_\Sigma(\cl(\Omc), \cl(\Omc))\equiv_{\Sigma\cup
\sigC(\mathcal{O})} \cl(\Omc)$. We note that, since $\Omc$ is in normal form, $\cl(\Omc)$ is obtained by replacing every literal $\quant r.A$ by some $\quant r.D$, where we also have the axioms $\neg D\sqcup A$. We observe furthermore that, since $\Omc$ is role isolated, the r-Rule is applicable in $\Omc$ exactly iff the r-Res is applicable on the corresponding normalized axioms, excluding the last premise. As a consequence, we obtain that 
$\Res_\Sigma(\cl(\Omc),\cl(\Omc))=\cl(\rolE_\Sigma(\Omc))$. We obtain $\cl(\rolE_\Sigma(\Omc))\equiv_{\Sigma\cup\NC}\cl(\Omc)$. Since we also have $\cl(\Omc)\equiv_{\Sigma\cup\NC} \Omc$, we obtain that $\rolE_\Sigma(\Omc)$ is a role forgetting of $\Omc$ for $\Sigma$.
\hideThisPart{

To show that $\rolE(\Omc)$ is a role forgetting for $\Omc$ and $\Sigma$, we need to show 
\begin{enumerate}
    \item $\sig(\rolE(\Omc))\subseteq\Sigma$,
    \item for every axiom $\alpha$ s.t. $\sig(\Omc)\subseteq\Sigma\cup\NC$ and $\Omc\models\alpha$, also $\rolE(\Omc)\models\alpha$,
    \item for every axiom $\alpha$ s.t. $\sig(\Omc)\subseteq\Sigma\cup\NC$ and $\rolE(\Omc)\models\alpha$, also $\Omc\models\alpha$.
\end{enumerate}
The first item follows from construction. For the third item, we observe that the r-Rule adds only axioms that follow from $\Omc$, and thus no new entailments can be created. What remains to be shown is the second item.

For simplicity, assume that $\sigC(\Omc)\subseteq\Sigma$. 
We furthermore assume $\Out_\Sigma\subseteq\Sigma$. \patrick{This assumption is not part of the theorem, but I think it has to be!}
We furthermore assume every $A\in\Out_\Sigma$ either occurs only in literals of the form $\exists r.A$, or $\forall r.A$ (that means, always under the same role and with the same quantifier).
\patrick{Again, this is not part of the theorem, but maybe should.
}
We then need to show that $\texttt{roleE}_\Sigma(\Omc)$ is a uniform interpolant of $\Omc$ for $\Sigma$. Since by construction, $\sig(\texttt{roleE}_\Sigma(\Omc))\subseteq\Sigma$, it only remains to show that $\texttt{roleE}_\Sigma(\Omc)\equiv_\Sigma\Omc$. For this, we take a route via concept forgetting. 

Define a new ontology $\Omc_2^\#$ as follows. 
\begin{enumerate}
    \item For every role name $r\in\sig(\Omc)\setminus\Sigma$, we introduce a fresh concept name $B_r$.
    \item For every concept name $A\in\Out(\Sigma)$, we introduce a fresh concept name $B_A$, for which we add the axiom $\neg B_A\sqcup A$.
    \item We add for every clause $\neg A_1\sqcup \ldots\neg A_n$ of the form (c1), where no $A_i$ occurs under an existential role restriction over a role outside of $\Sigma$, a new clause $\neg B_r\sqcup\neg B_{A_1}\sqcup\ldots\sqcup\neg B_{A_n}$.
    \item We add for every clause $\neg A_1\sqcup \ldots\neg A_n$ of the form (c1), where at least one  $A_i$ occurs under an existential role restriction over a role outside of $\Sigma$, a new clause $\neg B_{A_1}\sqcup\ldots\sqcup\neg B_{A_n}$.
    \patrick{Adapt the construction to not use $B_A$?}
    \item We replace every literal $\forall r.A$, where $r\not\in\Sigma$, by $\neg B_r\sqcup B_A$. 
    \item Every clause of the form $C\sqcup\exists r.D$, where $r\not\in\Sigma$, is replaced by two clauses of the form 
    $C\sqcup A_r$, $C\sqcup A_D$. We repeat this operation until no existential role restrictions over roles outside of $\Sigma$ remain. 
\end{enumerate}
Recall that Theorem~\ref{prop_A_rule} does not depend on other results shown in the appendix, so that we can use it already here. Using Theorem~\ref{prop_A_rule} we can conclude that $\texttt{conE}_{\Sigma}(\Omc^\#)\equiv_{\Sigma'}\Omc^\#$, where $\Sigma'=\Sigma\cup\sigC(\Omc)$. Intuitively, $\texttt{conE}_{\Sigma}(\Omc^\#)$ is the result of eliminating all fresh names in $\Omc^\#$. Moreover, we notice that the inference of the r-Rule used when computing $\texttt{rolE}_\Sigma(\Omc)$ directly correspond to inferences of the A-rule used when computing $\texttt{conE}_{\Sigma}(\Omc^\#)$, so that indeed, 
\[
     \texttt{conE}_{\Sigma}(\Omc^\#)\ =\  \texttt{rolE}_\Sigma(\Omc)
\]
We obtain that $\texttt{rolE}_\Sigma(\Omc)\equiv_\Sigma\Omc^\#$. What remains to be shown is that 
$\Omc^\#\equiv_\Sigma\Omc$, for which we use Lemma~\ref{{lem:inseparability}}. Specifically, this means that we have show for every model $\Imc$ of $\Omc$, how to obtain a model $\Imc^\#$ of $\Omc^\#$ s.t. $\Imc^\#=_\Sigma=\Imc$ and vice versa. 

\begin{enumerate}
    \item Let $\mathcal{I}^\# =  (\Delta, \ \cdot^{\mathcal{I}^\#})$ be a model of $\Omc^\#$. 
    We define an interpretation $\Imc=(\Delta,\cdot^\mathcal{I})$ s.t. $\Imc\models\Omc$ and $\Imc=_\Sigma\Imc^\#$. Since we want $\Imc^\#=_\Sigma\Imc$, we set $X^\Imc=X^{\Imc^\#}$ for all $X\in\Sigma$. 
    It remains to define $r^\Imc$ for $r\in\sig(\Omc)\setminus\Sigma$, for which we set 
    \[
        r^\Imc\ =\ \{(d,d) \mid d\in A_r^\Imc\} 
    \]
    It remains to show that $\Imc\models\Omc$. Looking at the transformation, we observe that for this, it is sufficient so show    
    $(\forall r.A)^\Imc=(\neg B_r\sqcup B_A)^{\Imc^\#}$ and $(\exists r.A)^\Imc=(B_r\sqcap B_A)^{\Imc^\#}$ for all $r\in\sigR(\Omc)\setminus\Sigma$ and $A\in\Out_\Sigma(\Omc)$.
    \patrick{We also have to check the other axioms that changed.}
    
    \begin{enumerate}
        \item Assume $d\in(\forall r.A)^\Imc$. Then, either $d$ has no $r^\Imc$-successors, in which case $d\in(\neg B_r)^{\Imc^\#}$, or all its $r^\Imc$-successors are in $A$. In the latter case, the only $r$-successor of $d$ is $d$ itself, and thus $d\in A^\Imc$, which means that also, $d\in A^\#$.  
        \patrick{Maybe the construction requires $B_A=A$.} 
        Conversely, assume $d\in(\neg A_r\sqcup A)^{\Imc^\#}$. Then, $d$ has either no $r^\Imc$-successor, or $d$ is the only $r$-successor, and we have $d\in A^\Imc$. In both cases, $d\in(\forall r.A)^\Imc$.
        \item Assume $d\in(\exists r.A)^\Imc$. The construction ensures that then, $d$ must be that $r$-successor, which means $(d,d)\in r^\Imc$, and $d\in A^\Imc$. By the construction, this means that $d\in(A_r\sqcap A)^{\Imc_\#}$.

        Conversely, if $d\in(A_r\sqcap A)^{\Imc_\#}$, then construction ensures directly that $d\in\{\exists r.A\}$.
    \end{enumerate}

    \item Let $\mathcal{I} =  (\Delta, \ \cdot^{\mathcal{I}})$ be a model of $\Omc$.

    
    We define an interpretation $\Imc^\#=(\Delta,\cdot^\mathcal{I^\#})$ s.t. $\Imc^\#\models\Omc^\#$ and $\Imc^\#=_\Sigma\Imc$. The latter is achieved by setting $X^{\Imc^\#}=X^\Imc$ for all $X\in\Sigma$.     
    For every $r\in\Sigma_r$, we set $B_r^{\Imc^\#}=(\exists r.\top)^\Imc$. For $B_A$, we make a case distinction based on the unique literal under which $A$ occurs in $\Omc$. If it is $\exists r.A$, then we set $(B_A)^{\Imc^\#}=(\exists r.A)^\Imc$. Otherwise, if it is 
    $\forall r.A$, then we set $(B_A)^{\Imc^\#}=(\forall r.A)^\Imc$.  
    %
    %
    To show that $\Imc^\#\models\Imc$, we need to again verify similar equalities between concept interpretations as for the previous item. However, this time, we exploit the fact that every $A\in\Out_\Sigma$ occurs in $\Omc$ either only 
    
    in literals of the form $\exists r.A$ or in literals of the form $\forall r.A$. This means that either only the interpretation of $\exists r.A$ or the interpretation of $\forall r.A$ is relevant for the satisfaction of the model.

    \begin{enumerate}
        \item Assume $d\in(\forall r.A)^\Imc$. Then, either $d$ has no $r^\Imc$-successors, or there are $r$-successors, and they are all in $A$. In the first case, we have $d\in (\neg A_r)^{\Imc^\#}$. In the latter case, we have $d\in(\exists r.A)^{\Imc}\subseteq B_A$. In both cases, we obtain $d\in(\neg B_r\sqcup B_A)^{\Imc^\#}$. 
        
        Assume $d\in(\neg B_r\sqcup B_A)^{\Imc^\#}$. Then, either $d\in(\neg B_r)^{\Imc^\#}=(\forall r.\bot)^\Imc\subseteq(\forall r.A)^\Imc$, or $d\in(B_A)^{\Imc^\#}$.  
        
        In the latter case, we obtain $d\in(\exists r.A)^\Imc$.  \patrick{and now?}
        
        \item Assume $d\in(\exists r.A)^\Imc$. Then, our construction ensures both $d\in B_r^{\Imc^\#}$ and 
        $d\in B_A^{\Imc^\#}$. 
        
        Conversely, if $d\in(B_r\sqcap B_a)^{\Imc^\#}$, then the construction ensures 
        $d\in((\exists r.\top\sqcap\forall r.A)\sqcup \exists r.A)^\Imc=(\exists r.A)^\Imc$.

    \end{enumerate}
    
\end{enumerate}

}
\end{proof}

%% file: appendix-properties.tex
\subsection{Proof of Proposition \ref{prop:size_gm}}

In the following, we use $\lvert\Omc\rvert$ to denote the number of axioms in
$\Omc$. We then have $\lvert\Omc\rvert\leq \lVert \Omc\rVert$.

\propsizegm*
\begin{proof}
We first show the upper bound.
Our construction ensures that for every axiom $\alpha \in \texttt{rolE}_\Sigma
(\OsplitForm)$, we can find a sequence of axioms $\beta_1,\ldots, \beta_n\in
\clSigmaPart$ such that $\alpha$ is obtained from $\bigsqcup_{1\leq i\leq
n}\beta_i$ by removing all literals that contain a role $r\not\in \Sigma$.
Because there are at most exponentially many subsets of $\clSigmaPart$,
this limits the number of possible inferred axioms to exponentially many.
We obtain
\begin{align*}
    \lvert \texttt{rolE}_\Sigma (\OsplitForm)\rvert &\leq 2^{\lvert
\clSigmaPart\rvert }\\
    &\leq 2^{\lVert \cl(\Omc)\rVert } \\
    \lVert \texttt{rolE}_\Sigma (\OsplitForm)\rVert &\leq \lvert
\texttt{rolE}_\Sigma (\OsplitForm)\rvert \cdot \lVert cl(\Omc)\rVert \\
    &\leq 2^{\lVert \cl(\Omc)\rVert } \cdot \Vert \cl(\Omc)\Vert .
\end{align*}

Similarly, for every axiom $\gamma = \texttt{conE}_\Sigma (\texttt{rolE}_\Sigma
(\OsplitForm))$,
we can find a sequence of axioms $\alpha_1,\ldots, \alpha_n\in
\texttt{rolE}_\Sigma (\OsplitForm)$ such that $\gamma$ is obtained from
$\bigsqcup_{1\leq i\leq n}\alpha_i$ by removing all literals of the form $A$ or
$\neg A$ with $A\not\in \Sigma$. As shown above, each $\alpha_k$ is obtained
from $\bigsqcup_{1\leq i\leq
n_k}\beta_i^k$, for some  $\beta_1^k,\ldots, \beta_{n_k}^k\in
\clSigmaPart$,  by removing all literals that contain a role name $r\not\in
\Sigma$. We obtain that $\gamma$  is obtained from $\bigsqcup\limits_{1\leq
k\leq n,1\leq i\leq
n_k}\beta_i^k$ by removing all literals $L$ such that  (i) $L$ contains a role
name $r\not\in \Sigma$, or (ii) $L$ is of the form $A$ or $\neg A$ with
$A\not\in \Sigma$. We obtain
\begin{align*}
\lvert \texttt{conE}_\Sigma  (\texttt{rolE}_\Sigma(\OsplitForm))\rvert
&\leq 2^{\lvert  \clSigmaPart\rvert } \\
& \leq 2^{\lVert \cl(\Omc)\rVert } \\
\lVert \texttt{conE}_\Sigma  (\texttt{rolE}_\Sigma(\OsplitForm))\rVert
&\leq \lvert \texttt{conE}_\Sigma  (\texttt{rolE}_\Sigma(\OsplitForm))\rvert
\cdot \lVert \cl(\Omc)\rVert  \\
&\leq 2^{\lVert \cl(\Omc)\rVert }\cdot \lVert \cl(\Omc)\rVert .
\end{align*}
For every definer $D\in\sig(\cl(\Omc))$, we have $\length{C_D}<\lVert
\cl(\Omc)\rVert $. Taking that the lengths of axioms in $\texttt{conE}_\Sigma
(\texttt{rolE}_\Sigma(\OsplitForm))$ are bound by $\lVert\cl(\Omc)\rVert$, we
obtain that for every axiom $\alpha\in \gm_\Sigma(\Omc)$, we
have $\length{\alpha}\leq \lVert \cl(\Omc)\rVert ^2$.  Consequently, we have
\begin{align*}
    \lVert \gm_\Sigma(\Omc)\rVert
    &\leq \lvert\texttt{conE}_\Sigma  (\texttt{rolE}_\Sigma(\OsplitForm))\rvert
    \cdot \lVert \cl(\Omc)\rVert^2 \\
    &\leq 2^{\lVert \cl(\Omc)\rVert }\cdot \lVert \cl(\Omc)\rVert ^2 \\
    &< 2^{3\cdot \lVert \cl(\Omc)\rVert }.
\end{align*}
Note that for the last step, we use that $n^2<2^{2n}$ for all integers $n\geq
0$.
This establishes the upper bound.

We continue to show the lower bound. %
For any integer $n\geq 0$, we define the ontology $\Omc_n$ to contain the
following axioms:
\begin{enumerate}
\item $Z_1\sqcap Z_2\sqcap \ldots\sqcap Z_n\sqsubseteq \bot$ (1 axiom of length
$n+1$),
\item $X_i\sqcup Y_i\sqsubseteq Z_i$ for all $1\leq i\leq n$ ($n$ axioms of
length 3),
\item $\top\sqsubseteq A_1\sqcup\exists s. X_1$,
$\top\sqsubseteq\overline{A_1}\sqcup \exists s. Y_1$ (2 axioms of length 4), and
\item $\top\sqsubseteq A_j\sqcup \forall s. X_j$,
$\top\sqsubseteq\overline{A_j}\sqcup \forall s.  Y_j$  for $2\leq j\leq n$ ($2n-2$
axioms of length $4$).
\end{enumerate}
As a signature, we define
 $\Sigma_n = \{A_{j}\}_{j=1}^n\cup\{\overline{A_{j}}\}_{j=1}^n.$

Normalizing $\Omc_n$ introduces the definers $D^X_i$ and $D^Y_i$ with
$C_{D^X_i} = X_i$ and
$C_{D^Y_i}=Y_i$. In particular, this gives $2n$ additional axioms of length 2
each, so that we obtain
$\lvert \cl(\Omc_n)\rvert = 5n + 1$ and
\begin{align*}
    \lVert\cl(\Omc_n)\rVert &= n+1 + 3n + 2\cdot 4 + (2n-2)\cdot 4 +2n\cdot 2 \\
                            &= 16n+1
\end{align*}
We continue to compute the sizes of the other axiom sets computed.
\begin{itemize}
\item $\mathcal{D}_{\Sigma_n}(\cl(\mathcal{O}_n))$
consists of axioms of the form:
$$
    \neg D_1^*\sqcup \ldots  \sqcup \neg D_n^*,\quad
    \text{where for }1\leq i\leq n,\quad D_i^*\in\{D^X_i, D^Y_i\}.
$$
We have $\lvert\mathcal{D}_{\Sigma_n}(cl(\mathcal{O}_n))\rvert =  2^{n}$ and
$\lVert\mathcal{D}_{\Sigma_n}(\cl(\mathcal{O}_n))\rVert =  n\cdot 2^{n}$.

\item $\texttt{rolE}_{\Sigma_n}(RI_{\Sigma^n}(\mathcal{O}_n))$ consists
of the axioms in $\mathcal{D}_{\Sigma_n}(\cl(\mathcal{O}^n))$ and axioms of the
form:
$$
    \neg A_1^*\sqcup \ldots  \sqcup \neg A_n^*,\quad
    \text{where for }1\leq i\leq n,\quad A_i^*\in\{A_i,
\overline{A_i}\}.
$$

We obtain that $\lvert\texttt{rolE}_{\Sigma_n}(RI_{\Sigma_n}(\mathcal{O}_n))\rvert=
2^{n+1}$ and  $\lVert\texttt{rolE}_{\Sigma_n}(RI_{\Sigma_n}(\mathcal{O}_n))\rVert=
n\cdot  2^{n+1}$.

\item Finally, since no definer occurs positively anymore, the axioms in
$\mathcal{D}_{\Sigma_n}(\cl(\mathcal{O}_n))$ are removed, so that we obtain
$\lvert\texttt{conE}_{\Sigma_n} (\texttt{rolE}_{\Sigma_n}
(RI_{\Sigma_n}(\mathcal{O}_n)))\rvert= 2^{n}$ and $\lVert\texttt{conE}_{\Sigma_n} (\texttt{rolE}_{\Sigma_n}
(RI_{\Sigma_n}(\mathcal{O}_n)))\rVert= n\cdot 2^{n+1}$.
\end{itemize}
As a final result, we obtain
$\lvert \gm_{\Sigma^n}(\Omc_n)\rvert = 2^n$ and $\lVert \gm_{\Sigma^n}(\Omc_n)\rVert =
n\cdot 2^{n+1}$.

To summarize, we defined a sequence of ontologies $\Omc_n$ with signatures
$\Sigma_n$ s.t.
$\lVert \cl(\Omc_n)\rVert = 16n+1$ and $\lVert \gm_{\Sigma^n}(\Omc_n)\rVert = n\cdot
2^{n+1}$. This establishes the second claim of the proposition.
\end{proof}
\patrick{Optically, the proof looks a bit ugly with all the complex formulas
inside text. Maybe this can be improved by pulling some more formulas out of
the text.}
\patrick{The proof looks good. However, I have one note: If we would assume in
our
normal form that no literal occurs twice in an axiom, this argument would
simplify a lot: there are only polynomially literals in $\cl(\Omc)$, each
literal occurs at most once per clause, ergo, there are only exponentially many
distinct axioms that can be expressed using those literals, which means that at
any place where we are normalized, we never have more than exponentially many
axioms. Substituting definers in the last step only gives a polynomial blow-up
in the end.

I did not find anything on such an assumption in the paper however. Nonetheless,
I am not sure whether you need this assumption for your rules to work correctly.
(All the consequence-based methods I know do that.)}

\subsection{Proof of Propositions~\ref{prop:gmMono} and~\ref{prop_specialO}}

To simplify the proofs of Propositions~\ref{prop:gmMono}
and~\ref{prop_specialO}, we define the
operator~$\texttt{defE}$, which applies the definer substitution explicitly.

\begin{definition}\label{def:defE}
Let $\Omc$ be an ontology that contains definers $D$, for which $C_D$ is
defined. Then, the \emph{definer substitution on $\Omc$} is the ontology
$\texttt{defE}(\Omc)$ that is obtained from $\Omc$ by replacing each definer $D$ by
the corresponding concept $C_D$.
\end{definition}

Since Proposition~\ref{prop_specialO} relies on a simpler situation
than Proposition~\ref{prop:gmMono}, namely where the input is normalized, it is
more convenient to start with it, before the more complex situation of
Proposition~\ref{prop:gmMono}.

\propspecialO*
\begin{proof}

We first make a general observation on the effect the definers have when computing
$\gm_\Sigma(\Omc)$ for normalized ontologies $\Omc$. First, since $\Omc$ is
normalized, only concept names occur under role restrictions, which means that
for every definer $D$ introduced, we have $C_D\in\NC$, and the only negative
occurrence of $D$ is in an axiom of the form $\neg D\sqcup C_D$. In case
$C_D\not\in\Sigma$, this means that previously eliminated concept names get
reintroduced by the definer substitution when computing $\gm_\Sigma(\Omc)$, but
they can only occur in two ways: 1) as negative literals in axioms of the
form $\neg C_D\sqcup C$, or 2) in literals of the form $\quant r.C_D$. This
also means that $\gm_\Sigma(\Omc)$ remains normalized. This means in particular
that $\Mmc$ contains no role name $r\not\in\Sigma$, and all concept names
$A\in\sigC(\Mmc)\setminus\Sigma$ occur either negatively or under role
restrictions.

Let $RI_\Sigma(\Mmc)$ be the \sigmaForm form for $\Sigma$  and $\Mmc$.
\hideThisPart{Since $\Omc$ is normalized
and $\mathcal{M} = \gm_\Sigma(\Omc)$, we have $C_D = A\in\NC$
for every definer $D\in \sig(cl(\Omc))$.}
We observe that $\sigR(\mathcal{M})\subseteq \Sigma$, since no role name outside
of $\Sigma$ is introduced to $\mathcal{M} = \gm_\Sigma(\Omc)$ when
substituting definers $D$ with their corresponding concepts $C_D$.
We obtain that $RI_\Sigma(\Mmc) = \cl(\Mmc)$.

By the definition of $\gm_\Sigma$, we thus have
$$\gm_\Sigma(\Mmc) =
\texttt{defE}(\texttt{conE}_\Sigma(\texttt{rolE}_\Sigma(cl(\mathcal{M})))).$$
We show that $\gm_\Sigma(\Mmc)=\Mmc$ using the following two results:   
\begin{enumerate}
    \item $\texttt{rolE}_\Sigma(\cl(\mathcal{M})) = \cl(\mathcal{M})$.

Since $\sigR(\mathcal{M})\subseteq \Sigma$, there is no role name to be
eliminated
by the operator $\texttt{rolE}$. This means that
$\texttt{rolE}_\Sigma(\cl(\mathcal{M})) =
\cl(\mathcal{M})$.

    \item $\texttt{defE}(\texttt{conE}_\Sigma(\cl(\mathcal{M}))) = \mathcal{M}$.

    First,  we show that
$\mathcal{M}\subseteq\texttt{defE}(\texttt{conE}_\Sigma(\cl(\mathcal{M})))$.
    Let $c\in\mathcal{M}$. Then, $c$ is of the form
 \begin{equation}\label{form_c}
         \neg B_1\sqcup\ldots\sqcup \neg B_k\sqcup\quant_1
r_1.A_1\sqcup\ldots\sqcup\quant r_n.A_n\sqcup C_1
 \end{equation}
    where 
    (i)~$B_i\in
        \sigC(\mathcal{O})\setminus \Sigma$ for $1\leq i\leq k$,
    (ii)~$A_i\in\sigC(\mathcal{O})\setminus\Sigma$ for $1\leq i\leq n$,
    and (iii)~$\sig(C_1)\subseteq\Sigma$.

    We show that $c\in\texttt{defE}(\texttt{conE}_\Sigma(\cl(\mathcal{M})))$.
    By the definition of $\mathcal{M} = \gm_\Sigma(\mathcal{O})$, we have the
following results.
    \begin{itemize}
        \item $c\in \mathcal{M}$ must be obtained from an axiom $c_1\in
\texttt{conE}_\Sigma(\texttt{rolE}_\Sigma(\OsplitForm))$ by
replacing every definer $D$ by the concept $C_D$. Then $c_1$ is of the form
        $$c_1=\neg D_1\sqcup \ldots \neg D_k\sqcup \quant_1 r_1.D_1'\sqcup
\ldots \quant_n r_n.D_n'\sqcup C_1,$$
        where
        $C_{D_i}=B_i$
        for $1\leq i\leq k$, and $C_{D_j'} = A_j$
for $1\leq j\leq n$.

      By our construction, $D_1$, $\ldots$, $D_k$ must also occur positively in
$\texttt{conE}_\Sigma(\texttt{rolE}_\Sigma(\OsplitForm))$, since
otherwise $c_1$ is deleted through $\texttt{conE}_\Sigma$. Consequently,  there
are $k$ axioms $c_2, \ldots, c_{k+1}\in
\texttt{conE}_\Sigma(\texttt{rolE}_\Sigma(\OsplitForm))$ that are of
the forms
        \begin{align*}
            c_2 &= C_2\sqcup \quant_1' r_1'. D_1,\\
            &\qquad \vdots\\
            c_{k+1} &= C_{k+1}\sqcup \quant_k' r_k'. D_k.
    \end{align*}

        \item $c_1$ must be obtained by applying A-Rules on $k+1$ axioms $c_1',
\ldots, c_{k+1}'\in \texttt{rolE}_\Sigma(\OsplitForm)$ that are of the forms
    \begin{align}
    c_1'&= \neg B_1\sqcup \ldots \sqcup \neg B_k\sqcup \quant_1 r_1.D_1'\sqcup
\ldots \quant_n r_n.D_n'\sqcup C_1,\label{eq:form2-start}\\
    c_2'&= \neg D_1\sqcup B_1\\
    &\vdots\\
    c_{k+1}'&=\neg D_k\sqcup B_k.\label{eq:form2-end}
    \end{align}

    \end{itemize}
    Modulo renaming of definers,\footnote{Note that it is in principle possible
that $\cl(\Mmc)$ contains more definers than $\cl(\Omc)$, since occurrences of
role restrictions can be multiplied. However, this does not affect the
following argument, since all definers get replaced by the same concept names
again.} we have
(i)~$c_1'\in cl(\mathcal{M})$ by  normalizing
$c$;
    (ii)~$c_2',\ldots,c_{k+1}'\in cl(\mathcal{M})$ by normalizing the axioms in
$\texttt{defE}(\{c_2, \ldots, c_{k+1}\})\subseteq \mathcal{M}$.
    Therefore, we can assume $\{c_1', c_2',\ldots, c_{k+1}'\}\subseteq
\cl(\mathcal{M})$. We obtain $c\in \gm_\Sigma(\mathcal{M})$ by repeating the
process of generating $c\in \gm_\Sigma(\mathcal{O})$ from $c_1',
c_2',\ldots, c_{k+1}'$. 
    As a result, we obtain that
$\mathcal{M}\subseteq\texttt{defE}(\texttt{conE}_\Sigma(cl(\mathcal{M})))$.
  Furthermore, we have
$\texttt{defE}(\texttt{conE}_\Sigma(\cl(\mathcal{M})))\subseteq \mathcal{M}$,
since all the axioms in $\cl(\mathcal{M})$ are of the
forms~\eqref{eq:form2-start}~--~\eqref{eq:form2-end}, and we
%
cannot obtain axioms other than $c$ in~\eqref{form_c} after applying the
operators $\texttt{conE}$ and $\texttt{defE}$. Consequently, we have
$\mathcal{M}=\texttt{defE}(\texttt{conE}_\Sigma(\cl(\mathcal{M})))$.
\end{enumerate}

From~1 and~2, it follows that $\mathcal{M} = \gm_\Sigma(\mathcal{M})$.
\end{proof}


\propgmMono*
\begin{proof}

Assume $\mathcal{M}_0 = \Omc$. We first show that every axiom $c\in
\mathcal{M}_i = \gm_\Sigma(\mathcal{M}_{i-1})$ is also in
$\mathcal{M}_{i+1}$ for all $i\geq 1$.

Any axiom $c\in\gm_\Sigma(\mathcal{M}_{i-1})$ is obtained from some axiom $c_{d}\in
\texttt{conE}_\Sigma(\texttt{rolE}_\Sigma(RI_\Sigma(\mathcal{M}_{i-1})
\big))$ by replacing every definer $D$ by the corresponding concept $C_D$.
Then, $\sig(c_{d})\subseteq \Sigma\cup \ND$ by the definitions of
$\texttt{rolE}_\Sigma$ and $\texttt{conE}_\Sigma$.
There are two different cases.
\begin{enumerate}
 \item $c_{d}$ does not contain negative definers. Then, $c_{d}$ and $c$ are
 of the forms
\begin{align}
    c_{d} &= C_1\sqcup \quant_1 r_1. D_1\sqcup\ldots\sqcup \quant_n r_n. D_n
    \label{eq:case1a}
    \\
    c &= C_1\sqcup \quant_1 r_1. C_{D_1}\sqcup\ldots\sqcup \quant_n r_n. C_{D_n}.
    \label{eq:case1b}
\end{align}
where, $\sig(C_1)\subseteq \Sigma$, $r_j\in \Sigma$ and $\quant_j\in
\{\exists, \forall\}$ for $1\leq j\leq n$. 

Then, $c\in\gm_\Sigma(\mathcal{M}_{i})$ because of the following
observations, which hold modulo renaming of definers.

\begin{itemize}
    \item $c_d\in \texttt{rolE}_\Sigma(RI_\Sigma(\mathcal{M}_i))$ because $c_{d}\in
RI_\Sigma(\mathcal{M}_i)$ and $\sig(c_d)\subseteq \Sigma\cup \ND$,
    \item $c_d \in
\texttt{conE}_\Sigma(\texttt{rolE}_\Sigma(RI_\Sigma(\mathcal{M}_i)))$ because
$c_{d}\in\texttt{rolE}_\Sigma(RI_\Sigma(\mathcal{M}_i))$, $\sig(c_d)\subseteq \Sigma\cup \ND$ and $c_{d}$ does not contain negative definers, and
    \item $\texttt{defE}(\{c_d\}) = \{c\}$ by definition.
\end{itemize}
%
\item $c_{d}$ contains negative definers. Then, $c_{d}$ and $c$ are respectively
of the forms
\begin{align}
    c_{d} &=  \neg D_1\sqcup\ldots\neg D_n\sqcup C_d'\label{eq:cdc_2a}\\
     c &= \neg C_{D_1}\sqcup\ldots\neg C_{D_n}\sqcup C',\label{eq:cdc_2b}
\end{align}
where $C_d', C'$ do not contain negative definers, and $C_d'$ and
$C'$ are of the forms as in~\eqref{eq:case1a} and~\eqref{eq:case1b}, respectively.

In this case, normalizing $c$ in $\mathcal{M}_i$ produces axioms in
$\cl(\mathcal{M}_i)$ that are different from $c_{d}$. However, we can still show
that $c\in \mathcal{M}_{i+1} = \gm_\Sigma(\mathcal{M}_i)$.
We only consider the case where $n=1$, that is, there is only one negative
definer in $c_d$. The case for $n>1$ is shown by repeating the argument
step-wise for each definer.
We distinguish 4 possible cases based on the syntactical shape of $C_{D_1}$.
\begin{enumerate}
    \item  $C_{D_1}$ is of the form $A$ or $\neg A$. This case is proved
similarly as for Proposition~\ref{prop_specialO}, where we considered the case
of normalized ontologies for which $C_{D_1}$ is always of the form
$A\in\NC$.
    \patrick{This could need some clarification.}

    \item  $C_{D_1} = \quant r. C_1$. We consider only the case where ${\quant}
= {\exists}$. The other direction is shown in a similar way by just switching
the quantifiers. 
With ${\quant}={\exists}$, we have
    \[
        c_d = \neg D_1\sqcup C_d',\qquad
        c = \neg (\exists r. C_1)\sqcup C'.
    \]
    Normalizing $c$, we obtain
    $$
        c_1'=\forall r. D_2'\sqcup C_{d}' \in cl(\mathcal{M}_i),
    $$
 where $C_{D_2'} = \neg C_1$. If $r\in \Sigma$, then we have $c\in
\gm_\Sigma(\Mmc_i)$ as in Case~1. Assume $r\not\in \Sigma$. We then make the
following observations:
    \begin{itemize}
        \item  There exists an axiom $$c_1=\neg D_1\sqcup \exists r. D_3'\in
RI_\Sigma(\mathcal{M}_{i-1})$$
         with $C_{D_3'} = C_1$. Here, $c_1$ is introduced when normalizing the
literal $C_{D_1} = \exists r. C_1$.
        \item There exists an axiom
         $$c_2 = C\sqcup \quant r.D_1\in
\texttt{conE}_\Sigma(\texttt{rolE}_\Sigma(RI_\Sigma(\mathcal{M}_{i-1}))
)$$
         for some $C$, $\quant$, $r$ because $D_1$ must also occur positively in
$\texttt{conE}_\Sigma(\texttt{rolE}_\Sigma(RI_\Sigma(\mathcal{M}_{i-1})))$. (Otherwise, $c_d$ will be deleted by $\texttt{conE}_\Sigma$.)
    \end{itemize}

    We have $c_1\in cl(\mathcal{M}_i)$ (modulo renaming of definers) by normalizing the axiom in
$\texttt{defE}(\{c_2\})\subseteq \mathcal{M}_i$.
    Furthermore, we have
    $$
        \cl(\mathcal{M}_i)\models D_2'\sqcap D_3'\sqsubseteq\bot.
    $$
  This allows us to make the following further observations.
 \begin{itemize}
     \item $c_d\in \texttt{rolE}_\Sigma(RI_\Sigma(\mathcal{M}_i))$ due to $c_{d}\in
\texttt{conE}_\Sigma(\texttt{rolE}_\Sigma(RI_\Sigma(\mathcal{M}_{i-1})))$ and 
the following inference with the r-Rule.
    $$\frac{\forall r. D_2'\sqcup C_d',\ \neg D_1\sqcup \exists r. D_3',\ \neg
D_2'\sqcup \neg D_3'}{\neg D_1 \sqcup C_d'}.$$
    \item $c_d\in
\texttt{conE}_\Sigma(\texttt{rolE}_\Sigma(RI_\Sigma(\mathcal{M}_i)))$ since $c_d\in
\texttt{rolE}_\Sigma(RI_\Sigma(\mathcal{M}_i))$ and $\sig(c_d)\subseteq \Sigma\cup
\textsf{N}_D$;
    \item $\{c\} =  \texttt{defE}(\{c_d\})$ as follows directly
from~\eqref{eq:cdc_2a}
and~\eqref{eq:cdc_2b}.
 \end{itemize}
 We obtain that $c\in \gm_\Sigma(\mathcal{M}_i, \Sigma) = \mathcal{M}_{i+1}$.

    \item $C_{D_1} = L_{1}\sqcap\ldots\sqcap L_{n}\sqcap \quant_{1}
r_{1}.C_{1}\sqcap \ldots \sqcap \quant_{m} r_{m}.C_{m}$, where $L_i = A_i$ or
$L_i=\neg A_i$ for some $A_i\in \NC$ and all $1\leq i\leq n$. Then, normalizing $c$ produces
following axiom in $\cl(\mathcal{M}_{i})$:
    $$\neg L_{1}\sqcup\ldots\sqcup \neg L_{n}\sqcup
    \quant_{1}^* r_{1}.D_{1}'\sqcup \ldots \sqcup
    \quant_{m}^* r_{m}.D_{m}'\sqcup C',$$
    where $C_{D_i'} = \neg C_i,\ \{\quant_i, \quant_i^*\} = \{\forall,
\exists\},\ 1\leq i\leq m$. As in Case~(b), we have
   \begin{itemize}
       \item $\neg D_1\sqcup L_1$, $\ldots$, $\neg D_1\sqcup L_n, \neg D_1\sqcup
\quant_1 r_1.D_1'$, $\ldots$, $\neg D_1\sqcup \quant_m r_m.D_m'   \in
RI_\Sigma(\mathcal{M}_i)$;
       \item $C\sqcup \quant r.D_1\in
\texttt{conE}_\Sigma(\texttt{rolE}_\Sigma(RI_\Sigma(\mathcal{M}_{i-1})
\big))$ for some $C$, $\quant$, $r$ because otherwise $c_d$ is removed in
Step~3 using $\texttt{conE}$.
   \end{itemize}
   We obtain that $c \in \gm_\Sigma(\mathcal{M}_{i})$ using the argument
from Case~(a) for every $L_i$ and from Case~(b) for every $D_i$.

    \item For the general case, we have $C_{D_1} = \overline{C_1}\sqcup \ldots
\sqcup \overline{C_n}$, where each $\overline{C_i}$ is as $C_D$ in Case~(c).
    In this case, we rewrite $c$ as $n$ different axioms.
    $$\overline{c_1}= \neg \overline{C_1}\sqcup C',
    \qquad \ldots, \qquad
    \overline{c_n} = \neg \overline{C_n}\sqcup C'.$$
For each $1\leq i\leq n$, we then have $\overline{c_i}\in
\gm_\Sigma(\mathcal{M}_{i})$ as in Case~(c).
\patrick{Is this really the most general case? Can there not be
arbitrarily nested occurrences of disjunctions and conjunctions? I right now
don't see why not --- but as a simple fix, I would just add a final case and
say it follows by induction by repeating the arguments for (c) and (d).}
\end{enumerate}


\end{enumerate}

We obtain in each case that $c\in \mathcal{M}_{i+1}$. As a consequence,
we have $\mathcal{M}_i\subseteq \mathcal{M}_{i+1}, \text{ for each }i\geq
1$.

It remains to show that there exists some $i_0\geq 0$ such that
$\mathcal{M}_{i_0}= \mathcal{M}_{i_0+1}$, since all axioms in $\mathcal{M}_{i}$
are axioms consisting of literals of the form 
$$A,\ \neg A,\ \neg C_D,\ Q r. C_D,  $$
where $C_D$ is a sub-concept of a concept in $\mathcal{O}$.
There exist only finitely many such literals and thus only finitely many such
axioms.
Consequently, the chain $\Mmc_0\subseteq\Mmc_1\subseteq\ldots$ must reach a
fixpoint after finitely many steps.
\end{proof}


%% file: appendix-optimizations.tex
\subsection{Proof of Theorem~\ref{cor:gmOpti}}

Recall the operators Op1, Op2 introduced in Section \ref{sec:elimD}. For
simplicity,
\begin{itemize}
\item let  $\mathcal{M}_1$ be the ontology obtained by applying the operator
Op1,
and
\item let $\mathcal{M}_2$ be the ontology obtained by applying the operator  Op2
on $\mathcal{M}_2$
\end{itemize}
\patrick{Using indices as in $\Mmc^n$ is not a good idea, since this usually
stands for exponentation.}

\begin{figure}
    \centering
   \begin{tcolorbox}
       \textbf{Role Propagation (RP)}:
   \begin{equation}\nonumber
      \frac{\bigcup\limits_{j=1}^{m}\{ P_j\sqcup C_j\},\qquad E_0\sqcup \quant
r. D_0,
      \qquad \bigcup\limits_{i=1}^{k}\{E_i\sqcup \forall r.
D_i\}}{(\bigsqcup_{i=0}^{n}E_i)\sqcup \quant r. (\bigsqcap_{j=0}^{m}C_j)},
    \end{equation}
    where  $P_0 = \bigsqcup_{i=0}^n\neg D_i$, for $j>0$, $P_j$ is a sub-concept
of $P_0$, $\quant\in \{\forall, \exists\}$, and $C_0$ and $C_j$ do not contain a
definer.

    \bigskip

  \textbf{Reduction (Red)}:
   \begin{equation}\nonumber
      \frac{\Omc\cup \{\neg D_1\sqcup\ldots\sqcup \neg D_n\sqcup C\}}{\Omc},
    \end{equation}
    where $C$ is a general concept expression that does not contain a negative
definer and $D_1,\ldots, D_n$ are definer symbols. The RP rule applies before
this rule if $\neg D_1\sqcup\ldots\sqcup \neg D_n$ takes the form of  $P_0$ in
the RP rule.
    \end{tcolorbox}
    \caption{Rules RP and Red.}
    \label{fig:UI-rule}
\end{figure}

We prove the correctness of the two operations one after the other. For the
first operation, we use a result from~\cite{DBLP:conf/dlog/SakrS22}, which
inspired our optimization, and which uses the inference rules shown in
Figure~\ref{fig:UI-rule}.
%
\begin{lemma}
$\texttt{conE}_{\Sigma}(\texttt{rolE}_{\Sigma}(\OsplitForm))\equiv_{\Sigma}
\mathcal{M}_1.$
\end{lemma}
\begin{proof}
Let $\mathcal{O}_{red},\ \Mmc_{red}$ be the ontologies obtained by applying the
rules in Figure~\ref{fig:UI-rule}
on $\texttt{conE}_{\Sigma}(\texttt{rolE}_{\Sigma}(\OsplitForm))$ and
$\mathcal{M}_1$ respectively.
By~\cite[Theorem 5]{DBLP:conf/dlog/SakrS22}, we have
\begin{align*}
\texttt{conE}_{\Sigma}(\texttt{rolE}_{\Sigma}(\OsplitForm))&\equiv_{\Sigma}
\mathcal{O}_{red},\\
\mathcal{M}_1&\equiv_{\Sigma}\mathcal{M}_{red}.
\end{align*}

Because the rule conD-Elim used for computing $\Mmc_1$ is a special case
of the RP rule, we have $\mathcal{O}_{red}= \Mmc_{red}$.
As a consequence, we obtain
$\texttt{conE}_{\Sigma}(\texttt{rolE}_{\Sigma}(\OsplitForm))\equiv_{\Sigma}
\mathcal{M}_1.$
\end{proof}

\begin{lemma}
$ \mathcal{M}_1\equiv_{\Sigma} \mathcal{M}_2.$
\end{lemma}
\begin{proof}
It is shown in \cite[Theorem1]{DBLP:conf/ijcai/ZhaoS17}, as an easy consequence
of Ackermann's lemma, that if the set of all negative occurrencess of a
definer~$D$ in $\Omc$ is of the form $\{\neg D\sqcup C_j\mid 1\leq j\leq n\}$,
then we can replace all positive occurrences of $D$ in $\Omc$ by
$\bigsqcap_{j=1}^{n}C_j$ without losing logical consequences that do not
involve $D$. This is exactly what Op2 does.
\end{proof}

\corgmOpti*

\begin{proof}
This can be shown almost in the same way as in the proof for
Theorem~\ref{main_theo}. The only difference is that we change the definition of
$O_D$ in Equation~\eqref{O-D} to
$$\mathcal{O}_D=\{ D\equiv \texttt{copy}_\Sigma(C_D)\mid D\in
\sig(\mathcal{M}_2)\cap \ND\}.$$
The reason is that we do not need to consider definers eliminated by the rules
con-Elim and D-prop anymore.
\end{proof}

%% file: appendix-classical-modules.tex
 \subsection{Proof of Theorem \ref{theo:deduct}}
 \theodeduct*
 \begin{proof}
 Set $\Mmc = \Dm_\Sigma(\Omc)$. We note that
$\gm_\Sigma(\Mmc)=\gm_\Sigma(\Omc)$, since $\Mmc$
contains exactly the set of axioms that are used to compute
$\gm_\Sigma(\Omc)$.
By Theorem~\ref{main_theo}, we have $\Mmc\equiv_\Sigma\gm_\Sigma(\Mmc)$ and
$\Omc\equiv_\Sigma\gm_\Sigma(\Omc)$. Putting these observations
together, we obtain $\Mmc\equiv_\Sigma\Omc$. Since  $\Mmc\subseteq\Omc$,
$\Mmc$ is a deductive module of $\Omc$ for $\Sigma$.
%
\end{proof}